%% file: neurips_2026.tex
\documentclass{article}

\usepackage[main, final]{neurips_2026}

\usepackage[utf8]{inputenc} 
\usepackage[T1]{fontenc}    
\usepackage{hyperref}       
\usepackage{url}            
\usepackage{booktabs}       
\usepackage{amsfonts}       
\usepackage{nicefrac}       
\usepackage{microtype}      
\usepackage{xcolor}         
\usepackage{amsmath}

\usepackage{amsmath}
\usepackage{amssymb}
\usepackage{amsthm}
\usepackage{graphicx} 

\usepackage{algorithm}
\usepackage{algpseudocode}
\usepackage{wrapfig}
\usepackage{placeins}
\usepackage[numbers, sort]{natbib}

\usepackage{subcaption}    
 \usepackage{multirow}

\newtheorem{theorem}{Theorem}[section]          
\newtheorem{lemma}[theorem]{Lemma}              
\newtheorem{remark}{Remark}[section]            

\hypersetup{
    colorlinks=true,
    urlcolor=blue,
    linkcolor=black,
    citecolor=blue
}

\title{Mechanism‑Aware Ensemble Conditioning for Data‑Limited Emulation of Extreme Events}

\author{%
  \begin{tabular}{cc}
    \textbf{Isabella S. Thiel}\textsuperscript{1}
    &
    \textbf{Juan Bello-Rivas}\textsuperscript{2}
    \\
    \texttt{thieli@mit.edu}
    &
    \texttt{jmbr@superadditive.com}
    \\[1em]
    \textbf{Yannis G. Kevrekidis}\textsuperscript{3}
    &
    \textbf{Themistoklis P. Sapsis}\textsuperscript{1}
    \\
    \texttt{yannisk@jhu.edu}
    &
    \texttt{sapsis@mit.edu}
  \end{tabular}%
}

\begin{document}

\maketitle

\begingroup
\renewcommand{\thefootnote}{\arabic{footnote}}

\footnotetext[1]{%
  Department of Mechanical Engineering,
  Massachusetts Institute of Technology,
  Cambridge, MA 02139, USA.%
}
\footnotetext[2]{%
  Department of Applied Mathematics and Statistics,
  Johns Hopkins University,
  Baltimore, MD 21218, USA.%
}
\footnotetext[3]{%
  Department of Chemical and Biomolecular Engineering,
  Johns Hopkins University,
  Baltimore, MD 21218, USA.%
}
\endgroup

\setcounter{footnote}{3}

\begin{abstract}
\input{0_abstract}
\end{abstract}

\input{1_introduction}

\input{3_method}

\input{4_lowdimexample}

\input{5_experiment}

\input{6_limitations}
\input{7_conclusion}


\begin{ack}
The research has been supported by the Vannevar Bush Faculty Fellowship N000142512059 as well as the AFOSR grant FA9550-23-1-0517.
The authors acknowledge the MIT Office of Research Computing and Data for providing high performance computing resources that have contributed to the research results reported in this paper.

\end{ack}

{
\bibliographystyle{unsrtnat}
\bibliography{references}
}

\appendix
\input{appendix}

\clearpage
\input{checklist.tex}

\end{document}

%% file: 0_abstract.tex
Extreme events in chaotic systems are difficult to learn from short trajectories because they are controlled by finite-time instabilities rather than by frequently observed bulk dynamics. 
We propose a \emph{mechanism-aware conditioning plug-in framework} that turns a nudged ensemble of an imperfect model into a non-intrusive sensor of the local instability geometry. 
In the small-noise regime, the ensemble covariance aggregates the same finite-time deformation kernels that govern local instability, providing a Jacobian-free proxy for the local amplification structure around a trajectory. 
A small Feature-Wise Linear Modulation (FiLM) network injects statistics of this ensemble geometry into an otherwise unchanged backbone without modifying the low-fidelity imperfect model simulator. 
We demonstrate this interface in two distinct settings: a Transformer-style residual-attention corrector for a controlled low-dimensional chaotic system and a probabilistic recurrent STORN corrector for topographic two-layer quasi-geostrophic (QG) flow. In the low-dimensional benchmark, leading ensemble-covariance directions identify transient instability episodes that are also detected by Optimally Time-Dependent (OTD) modes. FiLM conditioning improves 99th-percentile exceedance-frequency errors over an otherwise identical unconditioned Transformer. 
In QG, a fixed ensemble-conditioned FiLM-STORN model substantially improves long-horizon rare-event statistics in the data-limited regime, including density-tail errors, exceedance frequencies, and spatial exceedance-area distributions relative to an unconditioned STORN trained on the same data. On averaged high-threshold exceedance diagnostics, it also outperforms the unconditioned baseline STORN trained with $20$ times more expensive trajectory data.
These results show that local instability geometry is not merely interpretable post hoc, but an actionable conditioning signal for data-efficient rare-event emulation. Our code is available under
 \href{https://github.com/thisisisi/ExtremeEventSensor}{https://github.com/thisisisi/ExtremeEventSensor}.

%% file: 1_introduction.tex
\section{Introduction}

Extreme events in chaotic dynamical systems are difficult to emulate from short trajectories because they are not merely low-probability samples from a static distribution. They often arise during transient visits to locally unstable regions of the phase space, where a small number of finite-time growing directions dominate the dynamics and observables reach unusual values located in the tails of distributions \citep{sapsis2018,farazmand2019,sapsis2021}. A surrogate trained from limited data may reproduce the typical evolution accurately while missing the local dynamics that generate extremes. For applications across climate, rogue ocean waves, financial markets and risk-sensitive engineering systems, this is critical \citep{sapsis2021, lancelin2026a, fouque2011, cousins2016, dematteis2018, easterling2000} as accurate short-horizon predictions and low-order moments do not ensure accurate tail probabilities, return levels, and spatial organization of extremes \cite{farazmand2019, stamatelopoulos2026}.

Modern surrogate models for time-dependent PDEs have made substantial progress by exploiting spatial structures, operator learning, graph-based simulations, and hybrid physics--ML corrections \citep{li2021,lu2021,chen2025,wang2024}. A related but orthogonal line of work, such as statistics preserving learning, targets long-time statistical fidelity by training models to reproduce invariant measures or attractor statistics optimizing for global distributional losses rather than using local dynamical information \citep{jiang2024,hemmer2026, schiff2024}. Rare-event methods, including importance sampling, Adaptive Multilevel Splitting, large-deviation approaches, active learning, rare events sampling, and normalizing-flow-based samplers, provide powerful tools for estimating tail probabilities or transition pathways \citep{giardina2011,bouchet2019,dematteis2019,asghar2024,pickering2022, finkel2024}. These approaches are complementary to the setting considered here. We assume access to a computationally inexpensive but biased imperfect model simulator and only limited high-resolution reference or ground truth data, and we seek a fast non-intrusive correction model whose long-horizon statistics preserve rare-event structure.

Non-intrusive model-correction methods address this regime by nudging a imperfect model trajectory toward projected high-resolution data and learning a correction operator from the synchronized trajectory \citep{barthelsorensen2024,barthelsorensen2026,charalampopoulos2023}. Nudging resolves a key training difficulty by aligning the imperfect model and ground truth or reference states over finite windows, preventing chaotic exponential trajectory separation from corrupting their alignment. At deployment, the learned model maps low-fidelity model-state windows to corrected outputs without feeding them back into the imperfect model solver, leaving the underlying dynamics unchanged. However, existing nudging-based correction methods primarily use nudging for trajectory alignment, alongside its established role in data assimilation \citep{barthelsorensen2024,barthelsorensen2026,charalampopoulos2023,sun2019,huang2021}. They do not explicitly use the local ensemble geometry around the nudged trajectory as a conditioning signal, even though it contains information about the finite-time directions along which perturbations amplify.

We propose \emph{mechanism-aware ensemble conditioning}: a plug-in strategy that converts a nudged ensemble from an imperfect model into a local instability sensor and uses this signal to condition a temporal surrogate. Around the nudged imperfect model trajectory, we evolve a small stochastic ensemble using only the inexpensive and imperfect model simulator. In the small-variance regime, the ensemble covariance satisfies a perturbed Lyapunov equation whose homogeneous component is governed by the same deformation operators that define finite-time instability and Cauchy--Green geometry. The ensemble thus serves as a non-intrusive probe of local finite-time deformation, without requiring tangent dynamics, Jacobians, or adjoint solvers. Its purpose is to extract local instability information rather than to provide generic uncertainty estimates or generate additional high-resolution extreme events. Where covariance eigenspaces are retained, they can approximate instability directions under additional spectral-gap and scale-separation assumptions.
The resulting conditioning module is deliberately simple. Statistics of the local ensemble are passed through a small network that produces feature-wise linear modulation (FiLM) \citep{perez2017} parameters for a temporal backbone. The backbone architecture, the loss and its training set-up as well as the simulator remain otherwise unchanged. The contribution lies in using ensemble-derived deformation information to condition the surrogate. This isolates the core question: whether conditioning on the local geometry of the unstable subspace improves rare-event emulation under severe data constraints.

We instantiate this idea on a probabilistic recurrent correction model for a two-layer quasi-geostrophic (QG) system with additional localized topography-induced extremes. Training uses a short trajectory from a high-resolution simulation, while evaluation is performed on substantially longer trajectories. The ensemble-conditioned model improves tail-sensitive diagnostics and exceedance geometry relative to an unconditioned stochastic recurrent baseline, with the largest gains in the data-limited regime. A low-dimensional validation further supports the proposed mechanism. The leading eigenspaces of the ensemble covariance align with Optimally Time-Dependent (OTD) modes \citep{babaee2016, babaee2017} computed from tangent dynamics, supporting the interpretation of the ensemble as an instability proxy rather than a generic source of stochastic spread.

\vspace{-8pt}
\paragraph{Relation to prior work.}
Our setting differs from three adjacent lines of work. Neural operators, graph simulators, and coarse-correction networks learn state-to-state or field-to-field maps, often with strong spatial inductive bias, but they typically treat finite-time instability as implicit in the data rather than as an explicit conditioning signal \citep{li2021,lu2021,chen2025,wang2024,hoang2025}. Rare-event methods based on importance sampling, splitting, large deviations, optimal control, or normalizing flows are designed to estimate tail probabilities or transition pathways \citep{dematteis2018,giardina2011,bouchet2019,dematteis2019,asghar2024,giardina2006,noe2019}, whereas our goal is a fast non-intrusive surrogate for long-horizon corrected trajectories. Finally, EnKF-style ensembles and residual nudging estimate state uncertainty, forecast spread, or assimilation corrections \citep{kalman1960,houtekamer2016,luo2014}. Here, the ensemble serves as a controlled geometric probe whose covariance provides information about finite-time stretching, with connections to OTD modes and Cauchy--Green geometry \citep{babaee2016,babaee2017,blanchard2019a}.

Our contributions are:
\begin{itemize}
    \item  We formulate a nudged ensemble SDE whose covariance provides in the small-noise regime a non-intrusive proxy for finite-time instability geometry through a deformation kernel correspondence.
     \item We validate the geometric mechanism in a low-dimensional chaotic system by comparing ensemble covariance eigenspaces with OTD modes computed from tangent dynamics.
    \item We introduce a plug-in FiLM conditioning module that injects compact ensemble-derived instability features into temporal surrogate models without changing the low-fidelity imperfect model simulator or backbone architecture.
    \item We demonstrate improved data-limited rare-event emulation in low-dimensional and two-layer QG systems, for two distinct temporal backbones and tail-sensitive and exceedance diagnostics, supporting the interpretation of the method as a plug-in conditioning mechanism.
\end{itemize}

%% file: 3_method.tex
\section{Method: Non-intrusive ensemble conditioning}\label{sec:method}
\subsection{Nudged model dynamics and ensemble stochastic differential equation}\label{sec:method-sde}

Let the ground truth or reference system be \(\dot{\textbf{u}}(t)=F(\textbf{u}(t)), \ \textbf{u}(t)\in\mathbb{R}^N\), and let the imperfect model be \({\dot{v}(t)=f(v(t)), \ v(t)\in\mathbb{R}^n,\; n\ll N}\). A deterministic projection operator \(\mathcal P:\mathbb{R}^N\to\mathbb{R}^n\), restricts the ground truth to the model state space \(u(t)=\mathcal P\textbf{u}(t) \in \mathbb R^n \). Direct learning of \(v\mapsto u\) is ill-posed in chaotic regimes because trajectories separate exponentially and \(f\) is biased relative to \(\mathcal PF\) \citep{barthelsorensen2024,charalampopoulos2023, barthelsorensen2026}. We therefore define the nudged deviation \(q_\tau(t)\) by
\begin{equation*}
\dot{q}_\tau(t)=f\!\big(u(t)+q_\tau(t)\big)-\mathcal{P}F\!\big(\textbf{u}(t)\big)-\tau^{-1}q_\tau(t),
\label{eq:nudged}
\end{equation*}
where \(\tau>0\) is the nudging timescale. 
Nudging stabilizes the imperfect model's trajectory around the ground truth on finite windows, which makes local comparison meaningful. It is widely used in data assimilation and related settings (see e.g., \cite{barthelsorensen2024, sun2019, huang2021}). Around this nudged trajectory, we evolve an ensemble of $M$-perturbations. Each  ensemble member \(Q^{(m)}(t)\) denotes a deviation from the projected ground truth \(u(t)\). 

For $m = 1, \dots, M$, we define
\begin{equation}
\mathrm{d}Q^{(m)}(t)=
\Big(f\!\big(u(t)+Q^{(m)}(t)\big)-\mathcal{P}F\!\big(\textbf{u}(t)\big)-\tau^{-1}Q^{(m)}(t)\Big)\mathrm{d}t
+\sqrt{2\beta^{-1}}\mathrm{d}B_t^{(m)},
\label{eq:ensemble_sde}
\end{equation}
with independent Brownian motions \(B_t^{(m)}\) and inverse noise level \(\beta>0\). The ensemble is not used as uncertainty quantification. It is a controlled probe of local amplification directions around a synchronized base trajectory. Under standard global Lipschitz and linear-growth assumptions, the ensemble SDE is well posed with finite moments; see Appendix Lemma~\ref{app:wellposedness}.

\begin{figure}[h!]
    \centering
    \includegraphics[width=0.75\linewidth]{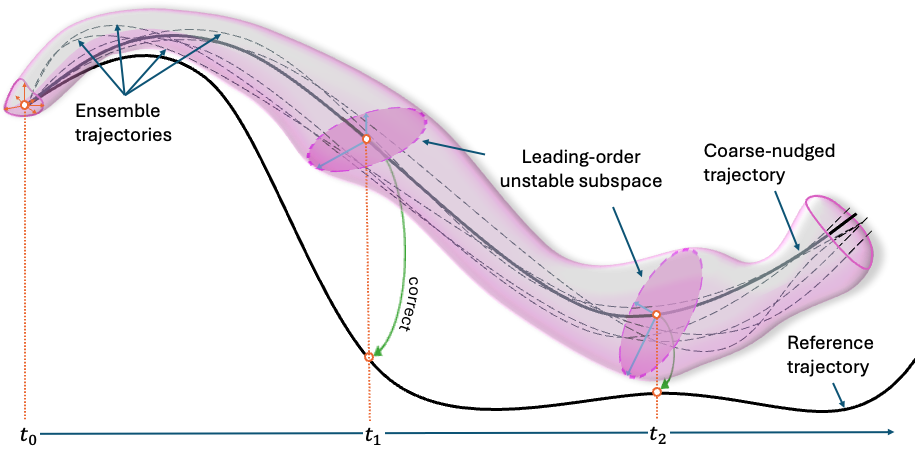}
    \caption{Schematic of the ensemble-based correction framework. At each time instant, a nudged trajectory of the imperfect model is constrained toward the ground truth trajectory used as reference. An ensemble of stochastically perturbed trajectories (dashed) is evolved around this nudged solution, whose spread at each instant approximates the leading-order unstable subspace (pink ellipsoids).}
    \label{fig:main_idea}
\end{figure}

\subsection{Covariance--instability correspondence}
Decompose each ensemble member by its mean contribution and its centered fluctuations, i.e. \(Q^{(m)}(t)=\bar{Q}(t)+Q^{\prime(m)}(t), \; \text{with } \bar{Q}(t)=\mathbb{E}[Q^{(m)}(t)]\).
The key observation is that in the small-noise regime the covariance evolves under the same linearized deformation operator that governs finite-time instabilities of the nudged dynamics. Thus its dominant directions encode a time-integrated finite-time stretching geometry. Under additional spectral-gap and scale-separation assumptions, these directions align with the leading Cauchy-Green eigenspaces. 

\begin{theorem}[Covariance--instability correspondence]\label{prop:covariance-instability}
Assume the conditions of Lemma~\ref{lem:wellposedness}. In addition, assume that \(\sup_{x\in\mathbb{R}^n}\|D^2 f(x)\|_{\mathrm{op}}\le H<\infty\) and that the centered fluctuations {\({Q'(t)=Q(t)-\bar Q(t)}\)}, where \(\bar Q(t)=\mathbb{E}[Q(t)]\), satisfy
\[
\sup_{t\in[0,T]}\mathbb{E}\|Q'(t)\|^2\le C_2\beta^{-1},
\qquad
\sup_{t\in[0,T]}\mathbb{E}\|Q'(t)\|^4\le C_4\beta^{-2}.
\]
Then, for all $t\in[0,T]$, the following hold:

\begin{enumerate}
\item[(i)] (\textbf{Mean-field consistency}) The ensemble mean satisfies
\[
\dot{\bar{Q}}(t) =f\big(u(t)+\bar{Q}(t)\big) - \mathcal{P}F\big(\textbf{u}(t)\big) - \tau^{-1}\bar{Q}(t) + O(\beta^{-1}).
\]
\item[(ii)] (\textbf{Covariance dynamics}) The covariance \(\Sigma(t) = \mathbb{E}\big[Q'(t)Q'(t)^\top\big]\) satisfies the perturbed Lyapunov equation
\begin{equation}
\label{eq:covariance_dynamics}
\dot\Sigma(t)=L_\tau(t)\Sigma(t)+\Sigma(t)L_\tau(t)^\top + \frac{2}{\beta}I +R_\Sigma(t),
\qquad
\|R_\Sigma(t)\|_{\mathrm{op}}\le C_T\beta^{-3/2}.
\end{equation}
where  \(L_\tau(t)=Df(u(t)+\bar Q(t))-\tau^{-1}I\) is the linearized nudged operator evaluated along the ensemble-mean state in the coarse or imperfect model, \(I\) is the identity matrix, and \(R_\Sigma(t)\) summarizes nonlinear remainder terms.
\item[(iii)] (\textbf{Deformation-kernel representation}) Consequently, if \(\Phi(t,s)\) is the fundamental solution generated by \(L_\tau(t)\), which maps linearized perturbations from time (s) to time (t), then
\[
\Sigma(t) = \Phi(t,0)\Sigma(0)\Phi(t,0)^\top + \frac{2}{\beta}\int_0^t \Phi(t,s)\Phi(t,s)^\top\,ds + O_T(\beta^{-3/2}).
\]

The first term propagates the initial covariance \(\Sigma(0)\), while the integral accumulates noise-forced deformation. When the initial-covariance contribution is negligible relative to the integral, the latter determines the leading covariance geometry at order \(O(\beta^{-1})\).
\end{enumerate}
\end{theorem}

\noindent \textit{Proof: see Appendix \ref{app:covariance-instability}}

Theorem \ref{prop:covariance-instability} is the central theoretical object in the method. It distinguishes the ensemble from ordinary uncertainty quantification. The ensemble is used because in the small-variance regime, spread inherits the geometry of finite-time stretching. This is not an ensemble Kalman filter. EnKF uses ensembles to estimate state uncertainty or posterior covariance under observational updates \cite{kalman1960, houtekamer2016}. Here, the ensemble is generated around a nudged trajectory to reveal the unstable directions that drive rare events and then turned into features for a surrogate. Likewise, the covariance is not interpreted as epistemic uncertainty over model parameters, but as a dynamical descriptor of local amplification structure.
Finite ensembles approximate the relevant subspaces with controlled sampling error. The covariance built from the centered fluctuations is a noise-driven probe of preceding deformation history. The next theorem quantifies how accurately a finite ensemble estimates the leading covariance eigenspace.

\begin{theorem}[Finite-sample approximation of covariance-induced instability subspaces] \label{theo:finite_ensemble_approx}
Fix \({t\in[0,T]}\). Assume that the centered ensemble fluctuations \(Q^{\prime(1)}(t),\ldots,Q^{\prime(M)}(t)\) are iid across members satisfying mean-zero \(\mathbb E[Q'(t)]=0\) and bounded fourth-moments \(\mathbb E[\|Q'(t)\|_2^4]\le \kappa<\infty\). 
Define the corresponding (empirical) covariance
\begin{equation*}
    \Sigma(t)=\mathbb E[Q'(t)Q'(t)^\top],\qquad \widehat\Sigma_M(t)=\frac1M\sum_{m=1}^M Q^{\prime(m)}(t)Q^{\prime(m)}(t)^\top .
\end{equation*}
Let \(R(t)\) and \(\widehat R(t)\) denote the orthogonal projectors onto the leading \(r\)-dimensional eigenspaces of \(\Sigma(t)\) and \(\widehat\Sigma_M(t)\), respectively. Here, \(r\) is the retained subspace dimension, and the eigenvalues \(\lambda_j\) are ordered from largest to smallest. Assume a spectral gap holds uniformly, i.e. \({\lambda_r(\Sigma(t))-\lambda_{r+1}(\Sigma(t))\ge \delta>0}\). Then
\begin{equation*}
    \mathbb E[\|\widehat R(t)-R(t)\|_{\mathrm{op}}]\le \frac{2\sqrt{\kappa}}{\delta\sqrt M}.
\end{equation*}
\vspace{-0.1cm}
Moreover, for every \(\varepsilon>0\),
\vspace{-0.1cm}
\begin{equation*}
    \mathbb P\{\|\widehat R(t)-R(t)\|_{\mathrm{op}}\ge\varepsilon\}\le \frac{4\kappa}{\delta^2M\varepsilon^2}.
\end{equation*}
\end{theorem}

\vspace{-0.25cm}

\noindent \textit{Proof: see Appendix \ref{app:finite_ensemble_approx}}

This Theorem is the finite-sample statement needed to justify the empirical use of the ensemble as a geometry estimator. The instability proxy is stable under sampling, and the error scales with the ensemble size, \(M^{-1/2}\) under standard assumptions. Combined with the previous covariance-instability Theorem \ref{prop:covariance-instability}, it yields the intended interpretation that for moderate \(M\), small \(\beta^{-1}\), and finite time horizons, the ensemble covariance identifies the local unstable subspace of the nudged flow. It is important to note that independence is required across ensemble members, not across time and that a small spectral gap weakens subspace identification, but does not invalidate the covariance--deformation relation established in Theorem~\ref{prop:covariance-instability}.

\subsection{Ensemble-conditioning module as a plug-in architecture}
Having identified the dynamically sensitive, unstable directions, we now use this information to condition an already established machine-learning corrector within a plug-in architecture.
Let \(c_t\) denote a low-dimensional summary of the ensemble over a short window  with duration $L$, i.e.,\([t-L,t]\),
\begin{equation*}
c_t=\phi\!\left(\{v_\tau^{(m)}(s)\}_{m=1}^M,\ s\in[t-L,t]\right).
\end{equation*}
In the high-dimensional quasi-geostrophic (QG) model, \(c_t\) is computed from compact statistics such as time-windowed means and standard deviations across ensemble members. This compression is intentional and avoids retaining and propagating the full ensemble and eigenspaces, but still encodes its local geometric structure in a lower-dimensional representation to identify time-localized dynamically sensitive regimes. The map \(g_\theta\) transforms the context \(c_t\) into Feature-Wise Linear Modulation (FiLM) \cite{perez2017} parameters,
\begin{equation*}
(\gamma_t,\beta_t)=g_\theta(c_t),
\end{equation*}
which modulate the hidden state \(h_t\) of a temporal backbone via
\begin{equation*}
\tilde{h}_t=(1+\gamma_t)\odot h_t+\beta_t.
\end{equation*}

\begin{figure}
    \centering
    \includegraphics[width=0.98\textwidth]{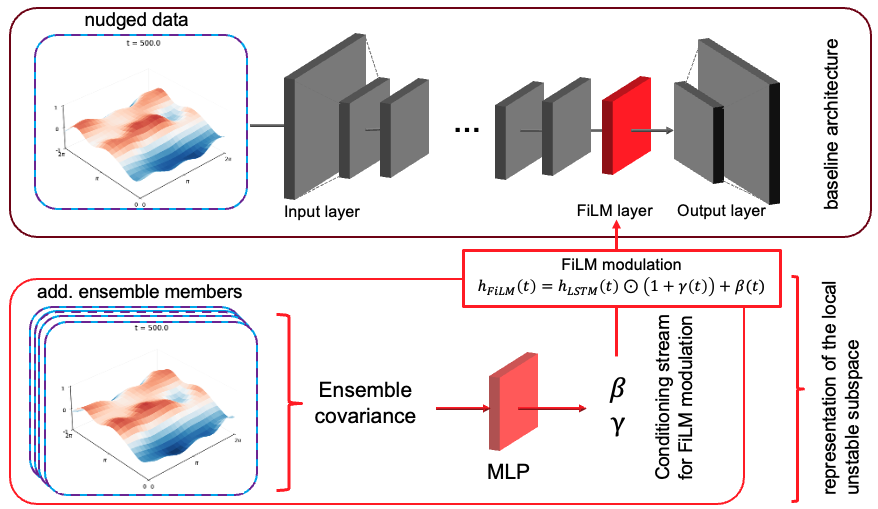}
    \caption{Plug-in backbone instantiation. The nudged ensemble summary entering a FiLM module attached to baseline network.}
    \label{fig:ML_plugin}
    \vspace{-5pt}
\end{figure}
\vspace{-0.12cm}
Here, \(\gamma_t\) and \(\beta_t\) are feature-wise scale adjustments and shifts, respectively, and \(\odot\) denotes element-wise multiplication. The FiLM shift \(\beta_t\) is distinct from the inverse noise intensity \(\beta\) in the ensemble SDE.
Importantly, this module is architecture-agnostic. It can be attached to a recurrent or a probabilistic sequence model, as well as a neural controlled differential equation, or a transformer, provided the backbone admits FiLM instantiation. Appendix \ref{app:lowdim_architecture} gives an end-to-end residual-attention instantiation for the three-dimensional toy system (see Section~\ref{sec:lowdimexample}); Section \ref{sec:experiment} and Appendix~\ref{app:ML_QG} use a probabilistic recurrent STORN backbone for QG flow (see Fig~\ref{fig:ML_plugin} for schematic implementation). In both cases, the shared mechanism is the FiLM conditioning on ensemble-derived instability summaries, not the temporal backbone itself.
Its role is to condition the surrogate on local instability geometry, not to replace the backbone or the simulator.
The central contribution is a data-efficient conditioning strategy that extracts ensemble-derived instability geometry to condition a learned surrogate. By leveraging latent geometric structure present in finite trajectory segments, the approach provides substantially richer information than individual trajectories, which is particularly valuable in data-scarce regimes where rare and extreme events are underrepresented. 
In contrast, we reinterpret a nudged ensemble as a dynamical sensor of local instability. Its covariance approximates the flow’s unstable geometry, and low-dimensional summaries of this structure condition the surrogate. This enables the model to exploit latent geometric information in the data, improving fidelity in distribution tails, especially under limited data.

%% file: 4_lowdimexample.tex
\section{Low‑dimensional toy problem}\label{sec:lowdimexample}

\begin{wrapfigure}[30]{r}{0.51\textwidth}
\vspace{-18pt}
    \centering
    \includegraphics[width=0.5\textwidth]{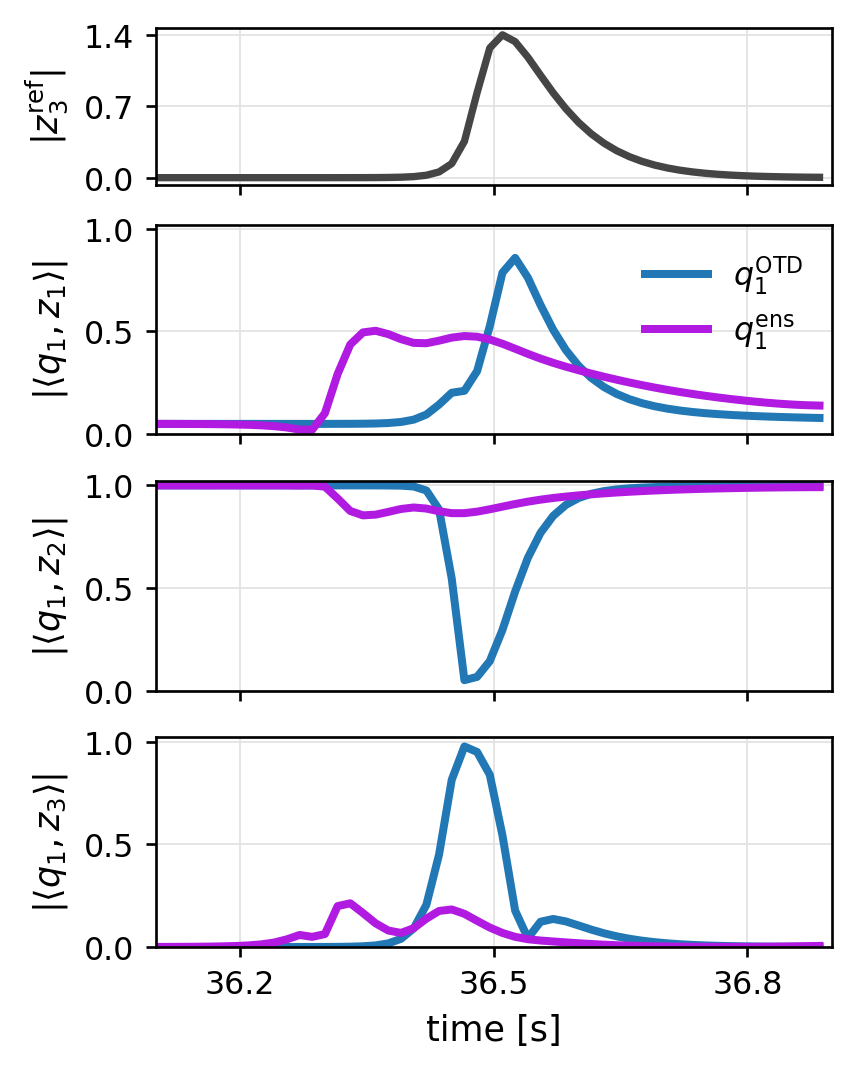}
    \caption{Validation of the ensemble-derived instability proxy comparing the component-wise magnitudes of the first OTD mode \(q_{\mathrm{OTD}}(t)\) and the leading ensemble-covariance mode \(q_{\mathrm{ens}}(t)\) near a finite-time instability burst (\(\tau=8.0\), \(\sigma=5\cdot 10^{-6}\)). The ensemble-derived mode activates over the same finite-time window as the OTD mode, indicating that the ensemble detects the relevant instability.}
    \label{fig:OTD_alignment}
\end{wrapfigure}

The three-dimensional system serves two roles. First, independently computed Optimally Time-Dependent (OTD) \cite{babaee2016, babaee2017} modes provide a tangent-based reference for testing whether the ensemble covariance detects finite-time instability episodes. Second, the system provides a controlled end-to-end correction problem in which ensemble-derived features condition a deterministic residual-attention corrector; details are given in Appendix \ref{app:lowdim_architecture}.

We consider the three-dimensional autonomous system of \cite{babaee2016}, which exhibits pronounced finite-time instabilities and well-defined OTD modes. Along a trajectory, we compute the leading OTD mode, hereafter refered to as \(q_{\mathrm{OTD}}(t)\) from tangent dynamics.
In parallel, we evolve the ensemble SDE around the same trajectory, as outlined in Section~\ref{sec:method-sde}, and compute the empirical covariance \(\hat{\Sigma}_M(t)\) of centered ensemble perturbations. We then comapre the leading eigenvector \(q_{\mathrm{ens}}(t)\) of \(\hat{\Sigma}_M(t)\) to \(q_{\mathrm{OTD}}(t)\) over time. See Appendix~\ref{app:lowdimsystem} for the full system definition and further details.

The time-resolved evolution between \(q^{{OTD}}(t)\) and \(q_{\mathrm{ens}}(t)\) remains consistent over extended intervals. This behavior is in agreement with the theoretical analysis in Section~\ref{sec:method}.
The comparison in Figure~\ref{fig:OTD_alignment} is intended to test temporal localization of finite-time instability rather than exact component-wise amplitude recovery. Perfect amplitude agreement across components is neither expected nor required. The relevant observation is that the ensemble-derived mode becomes active over the same burst window as the OTD mode. This supports the interpretation of the ensemble covariance as a non-intrusive detector of transient instability episodes rather than a generic measure of spread, and that this proxy can be constructed without computing tangent dynamics or Jacobians.

We further test whether ensemble conditioning improves the residual-attention corrector. Table~\ref{tab:lowdim_z3_sweep_summary} compares the unconditioned baseline with the FiLM settings that achieve the lowest error for each reported \(z_3\) tail diagnostic. A definition of the metrics is given in Appendix \ref{app:lowdim_metrics}. For notational simplicity, we identify $\sigma = \sqrt{2/\beta}$ as the addeed noise scale in all examples.

\vspace{-7pt}
\begin{table}[H]
\centering
\small
\caption{3D toy problem \(z_3\) tail diagnostics, with the best FiLM setting selected separately for each diagnostic from the parameter sweep. Entries are means \(\pm\) standard errors over five independent runs; lower values indicate smaller errors. Models were trained on a trajectory of length \(T_{\rm train}=150\) and evaluated on an independent trajectory of length \(T_{\rm test}=500\). This table summarizes sensitivity to the conditioning settings. For more detailed resuls see Tables~\ref{tab:lowdim_z3_sweep_summary} and \ref{tab:lowdim_attention_ablation}.
}
\label{tab:lowdim_z3_sweep_summary}
\begin{tabular}{llccc}
\toprule
Diagnostic & No-FiLM attention & Best FiLM setting & Best FiLM value & Reduction \\
\midrule
 $L_{1_{\log}}(z_3)$
& $6.903 \pm 0.100$
& eigvec, $\sigma=10^{-6}$, $M=2$
& $6.727 \pm 0.067$
& $2.5\%$ \\
 $E_{99}(z_3)$
& $0.0056 \pm 0.0017$
& mean/std/eigs, $\sigma=10^{-5}$, $M=2$
& $0.0041 \pm 0.0016$
& $26.3\%$ \\
\bottomrule
\end{tabular}

\end{table}

%% file: 5_experiment.tex
\vspace{-0.5cm}
\section{Experiments on high-dimensional chaotic PDE (QG)}\label{sec:experiment}

\subsection{Two-layer QG setup and nudged ensembles}

We evaluate our method on the baroclinic two-layer quasi-geostrophic (QG) system, a canonical high-dimensional chaotic PDE with intermittent and topography-induced extremes \cite{qi2018, barthelsorensen2024, barthelsorensen2026}. The ground truth or reference dynamics are simulated on a \(128\times128\) grid while the imperfect model uses the same equations on a coarser \(24\times24\) grid. The dominant error is the unresolved-scale bias rather than a changed governing law. A fixed spectral projector maps the ground truth state to the coarser model space, and lower-layer topography induces spatially localized exceedances, as shown in Figure~\ref{fig:app_QGtopograpy}. The nudged ensemble is generated by relaxing the imperfect model trajectory toward the projected ground truth on a timescale \(\tau\) and injecting isotropic perturbations of amplitude \(\sigma\) only in resolved modes, so the ensemble remains interpretable as a local probe around a synchronized imperfect model trajectory rather than as a generic stochastic simulator. We refer to the Appendix~\ref{app:QG} for a complete description of the governing equations, numerical discretization, and simulation parameters.

\subsection{Baseline}
The primary baseline is a probabilistic recurrent correction model (STORN) trained on a single nudged trajectory, without ensemble inputs. This choice is based on prior controlled comparisons on the same system, where standard RNNs, VAE–RNN hybrids, and variational RNNs were evaluated under identical conditions and probabilistic recurrent formulations consistently performed best \cite{barthelsorensen2024, barthelsorensen2024a}.

We evaluate the corrected trajectory against the high-resolution ground truth using Kullback–Leibler (KL) divergence to capture attractor misalignment and the log-L1 distance to emphasize tail discrepancies. We also evaluate exceedance-area distributions and event frequencies. Full details and a definition of the metrics are provided in the Appendix~\ref{app:evaluation}.

\subsection{Instability-aware training and data-limited learning of extremes}

For FiLM-STORN, the training context contains four scalars: the ensemble mean and standard deviation of the spatially averaged absolute streamfunction in each layer,
\begin{equation*}
\begin{aligned}
a_\ell^{(m)}(t)
&= \frac{1}{|\Omega|}
   \int_\Omega |\psi_\ell^{(m)}(x,t)|\,\mathrm{d}x,
\qquad \ell\in\{1,2\},\\
\mu_\ell(t)
&= \frac{1}{M}\sum_{m=1}^{M}a_\ell^{(m)}(t),
\qquad
s_\ell(t)
= \left[
  \frac{1}{M-1}\sum_{m=1}^{M}
  \bigl(a_\ell^{(m)}(t)-\mu_\ell(t)\bigr)^2
  \right]^{1/2},\\
c_t
&= \bigl(\mu_1(t),s_1(t),\mu_2(t),s_2(t)\bigr).
\end{aligned}
\end{equation*}
Here, $\Omega$ is the spatial domain and $M$ is the number of ensemble members used to compute the statistics. The full nudged field of the imperfect model is supplied separately to the STORN backbone. The context describes layer amplitudes and ensemble spread. It does not retain the full covariance eigenspace.
At deployment, the context is computed solely from the free-running imperfect model trajectory. The high-resolution reference is used only for offline evaluation. It supplies neither inputs nor context to the deployed model. No nudging or ensemble is used, and corrected outputs are not fed back into the imperfect model solver. Appendix~\ref{app:ML_QG} outlines the deployment-context definition.

The central result is that ensemble conditioning improves tail fidelity in the data-limited regime. With short high-resolution training windows, the conditioned model better matches the ground truth on tail distributions, exceedance probabilities, and long-horizon density statistics than the single-trajectory baseline trained on the same data. In several settings, it matches or exceeds a baseline trained on much longer trajectories, which is the relevant comparison for data efficiency in rare-event emulation. Tail- and exceedance-specific diagnostics are reported in Table~\ref{tab:aot_geometry_main}.

Additional controls and ablations separate the conditioning information from the way it enters the network (Appendix~\ref{app:ml:ablation}). At the fixed main-paper configuration, covariance-direction conditioning through FiLM yields lower errors than a dimension- and norm-matched random direction on all six diagnostics. FiLM also outperforms concatenation for both covariance-direction and mean/spread inputs. Neither the mean-only nor the variance-only context matches the combined mean/spread context across all diagnostics.

\begin{table}[t] 
\centering
\small
\caption{
QG comparison for high-threshold extreme-event diagnostics averaged over five independent runs. Lower is better for all metrics. The STORN and FiLM-STORN models use the same previously validated backbone \cite{barthelsorensen2024}, loss, optimizer and all hyperparameters are taken from the long-data state-of-the-art literature configuration \cite{barthelsorensen2024} STORN$^*$ model with $T_{\mathrm{train}}=1000$. The FiLM differs only by the addition of the ensemble-conditioned module using parameters $(\sigma, M)=(0.01,2)$. With only $T_{\mathrm{train}}=50$, FiLM-STORN improves all extreme event diagnostics relative to STORN$^*$ trained on 20 times longer data. A full sweep over $\sigma$ and $M$ is reported in Appendix~\ref{app:addresults:details}.
}
\label{tab:main_distributional_summary}
\setlength{\tabcolsep}{3.2pt}
\renewcommand{\arraystretch}{1.15}
\begin{tabular}{@{}l c c r r r r r r@{}}
\toprule
Method
& $T_{\mathrm{train}}$
& $(\sigma,M)$
& $D_{\mathrm{KL}}\downarrow$
& $L_{\log}\downarrow$
& \shortstack{$\psi_2\,\bar E_{\mathrm{AOT}}$\\$\downarrow$}
& \shortstack{$\psi_2\,\bar E_{\mathrm{freq}}$\\$\downarrow$}
& \shortstack{Avg. $\bar E_{\mathrm{AOT}}$\\$\downarrow$}
& \shortstack{Avg. $\bar E_{\mathrm{freq}}$\\$\downarrow$} \\
\midrule
STORN$^*$
& 1000
& \textemdash
& \textbf{0.00393}
& \textbf{5.3454}
& 0.00501
& 0.10564
& 0.01195
& 0.11726 \\
STORN
& 50
& \textemdash
& 0.07079
& 11.970
& 0.00589
& 0.14660
& 0.01574
& 0.12118 \\
FiLM-STORN (ours)
& 50
& $(0.01,2)$
& 0.03646
& 8.2959
& \textbf{0.00391}
& \textbf{0.06874}
& \textbf{0.01047}
& \textbf{0.07199} \\
\bottomrule
\end{tabular}
\vspace{-5pt}
\end{table}

\subsubsection{Ensemble parameter sweeps and robustness}

We vary the noise amplitude \(\sigma\) and ensemble size \(M\) to assess sensitivity to the ensemble configuration (Figure~\ref{fig:noise_ensemble_metrics}). Performance is strongest over an intermediate range of noise amplitudes. This pattern is consistent with a trade-off between obtaining an informative perturbation response and keeping the ensemble close to the nudged trajectory  and the ensemble not merely being a generic stochastic regularizer (see Figure \ref{fig:noise_ensemble_metrics}). In detail if \(\sigma\) is too small, the ensemble collapses and fails to resolve anisotropy and localize finite-time instability episodes, if \(\sigma\) instead is too large, the covariance no longer tracks the local nudged dynamics. Increasing \(M\) improves performance only until the dominant covariance directions are stably resolved, after which the gains saturate.

\begin{figure}[htbp]
    \centering
     \includegraphics[width=1.0\linewidth]{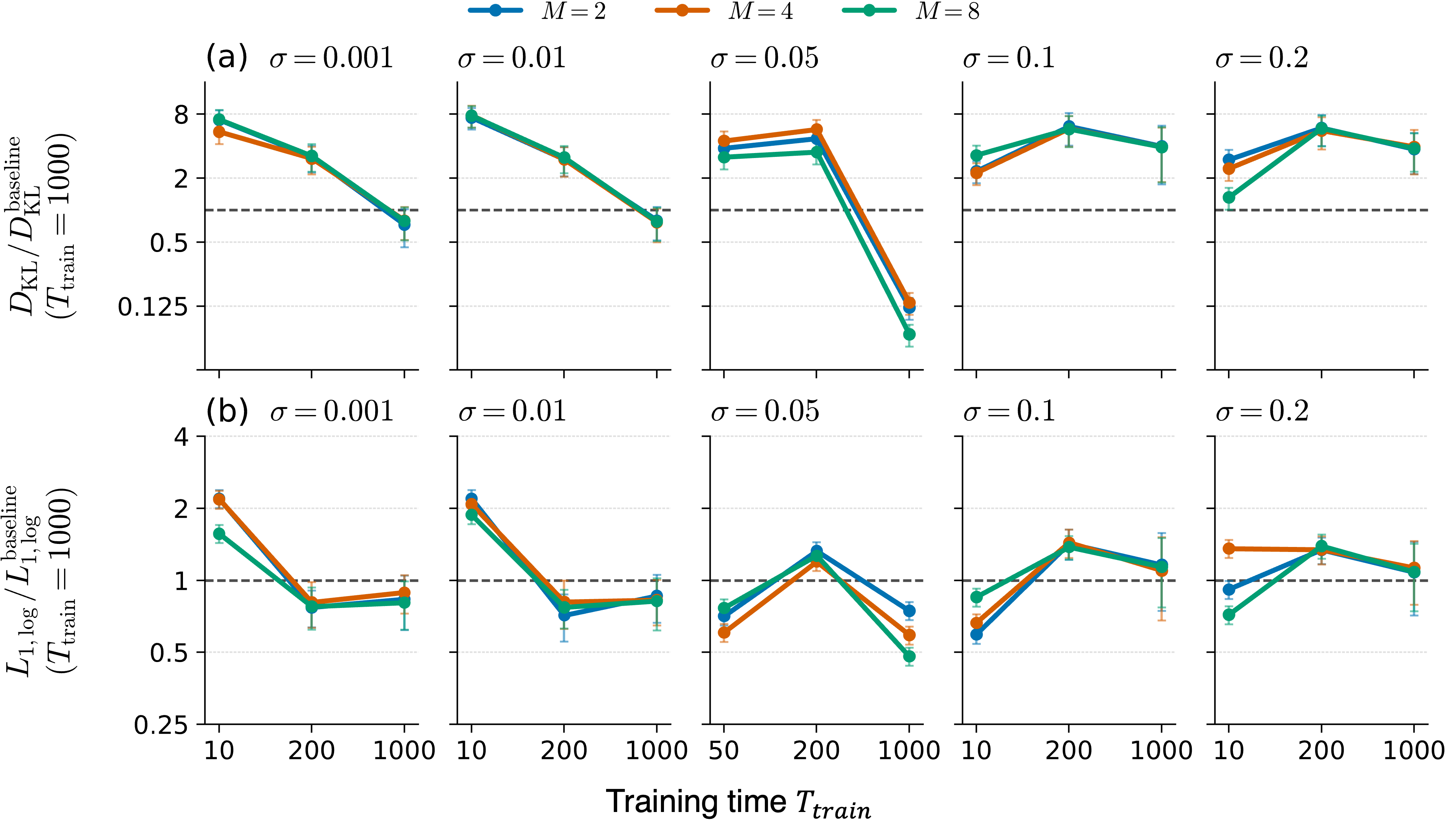}
    \caption{Distributional errors across training lengths, noise amplitudes, and ensemble sizes: (a) KL divergence and (b) log-\(L_1\) error. Both metrics are normalized by the corresponding error of the STORN baseline trained with \(T_{\mathrm{train}}=1000\); values below one indicate lower error than this baseline.}
    \label{fig:noise_ensemble_metrics}
    \vspace{-5pt}
\end{figure}

\subsubsection{Extreme-event structure and exceedance geometry}
We next examine the spatial extent and frequency of threshold exceedances. Beyond scalar tail statistics, the conditioned model better reproduces the spatial structure of exceedance sets.
For each threshold \(c\), the normalized exceedance area \(A_c(t)/A\) is the fraction of the domain above that threshold, where \(A_c(t)\) is the exceedance area and \(A\) is the total domain area. Its distribution measures how much of the domain is affected by an extreme, while \(\mathbb P(A_c(t)>0)\) measures how often any exceedance above a threshold $c$ happens.  In the QG benchmark, extreme events are localized and anisotropic, with dependence on the introduced bottom topography (see Fig \ref{fig:app_QGtopograpy}). A model that matches marginal tails but misses this geometry is therefore not capturing the mechanism of extremes. The ensemble-conditioned surrogate improves the localization, shape, and persistence of thresholded event regions relative to the single-trajectory baseline. This matters because the relevant object is not just the distribution of amplitudes but the geometry of the excursion manifold that generates those amplitudes. The empirical pattern therefore supports that local instability information, extracted non-intrusively from a nudged ensemble, can be converted into a conditioning signal that improves rare-event pathways, not only pointwise prediction.  Figure~\ref{fig:aots} compares the exceedance-area distributions, including their tails at high thresholds.

\begin{table}[t]
\centering
\small
\caption{
Extreme-event geometry errors over high thresholds \(c\in\{1.0,1.5,1.75\}\) and averaged over five independent runs. 
\(\overline{\mathcal E}_{\mathrm{AOT}}\) is the mean empirical Wasserstein-1 distance between the distributions of exceedance area \(A_c(t)/A\), and \(\overline{\mathcal E}_{\mathrm{freq}}\) is the mean absolute error in event frequency \(\mathbb P(A_c(t)>0)\). Lower is better. The lower layer \(\psi_2\) is reported separately because the topographic forcing acts directly in that layer; the two-layer average checks that the improvement is not isolated to one field. FiLM-STORN uses \(T_{\mathrm{train}}=50\), \(\sigma=0.01\), and \(M=2\).
}
\setlength{\tabcolsep}{5pt}
\renewcommand{\arraystretch}{1.12}
\begin{tabular}{cccccc}
\toprule
Diagnostic & Metric & CR & STORN \(T=50\) & STORN \(T=1000\) & FiLM-STORN \(T=50\) \\
\midrule
\(\psi_2\)
& \(\overline{\mathcal E}_{\mathrm{AOT}}\)
& 0.00970 & 0.00589 & 0.00501 & \textbf{0.00391} \\
\(\psi_2\)
& \(\overline{\mathcal E}_{\mathrm{freq}}\)
& 0.25215 & 0.14660 & 0.10564 & \textbf{0.06874} \\
Avg. \(\psi_1,\psi_2\)
& \(\overline{\mathcal E}_{\mathrm{AOT}}\)
& 0.03115 & 0.01574 & 0.01195 & \textbf{0.01047} \\
Avg. \(\psi_1,\psi_2\)
& \(\overline{\mathcal E}_{\mathrm{freq}}\)
& 0.40190 & 0.12118 & 0.11726 & \textbf{0.07199} \\
\bottomrule
\end{tabular}
\label{tab:aot_geometry_main}
\vspace{-8pt}
\end{table}

\begin{figure}[htbp]
    \centering
    \begin{subfigure}[b]{0.48\linewidth}
        \centering
        \includegraphics[width=1.07\linewidth]{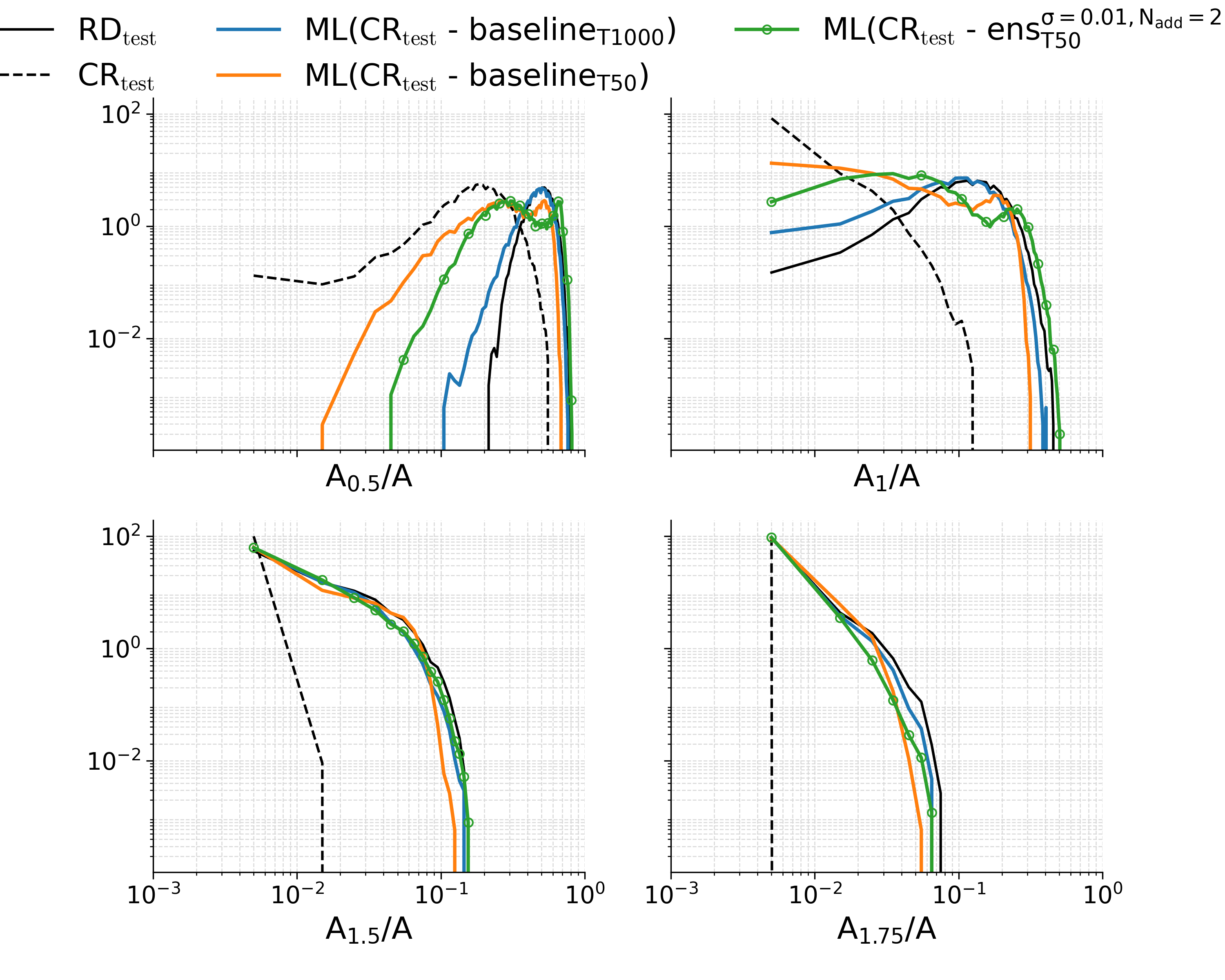}
        \caption{Threshold exceedance: $\psi_1$}
        \label{fig:psi1}
    \end{subfigure}
    \hfill
    \begin{subfigure}[b]{0.48\linewidth}
        \centering
        \includegraphics[width=1.07\linewidth]{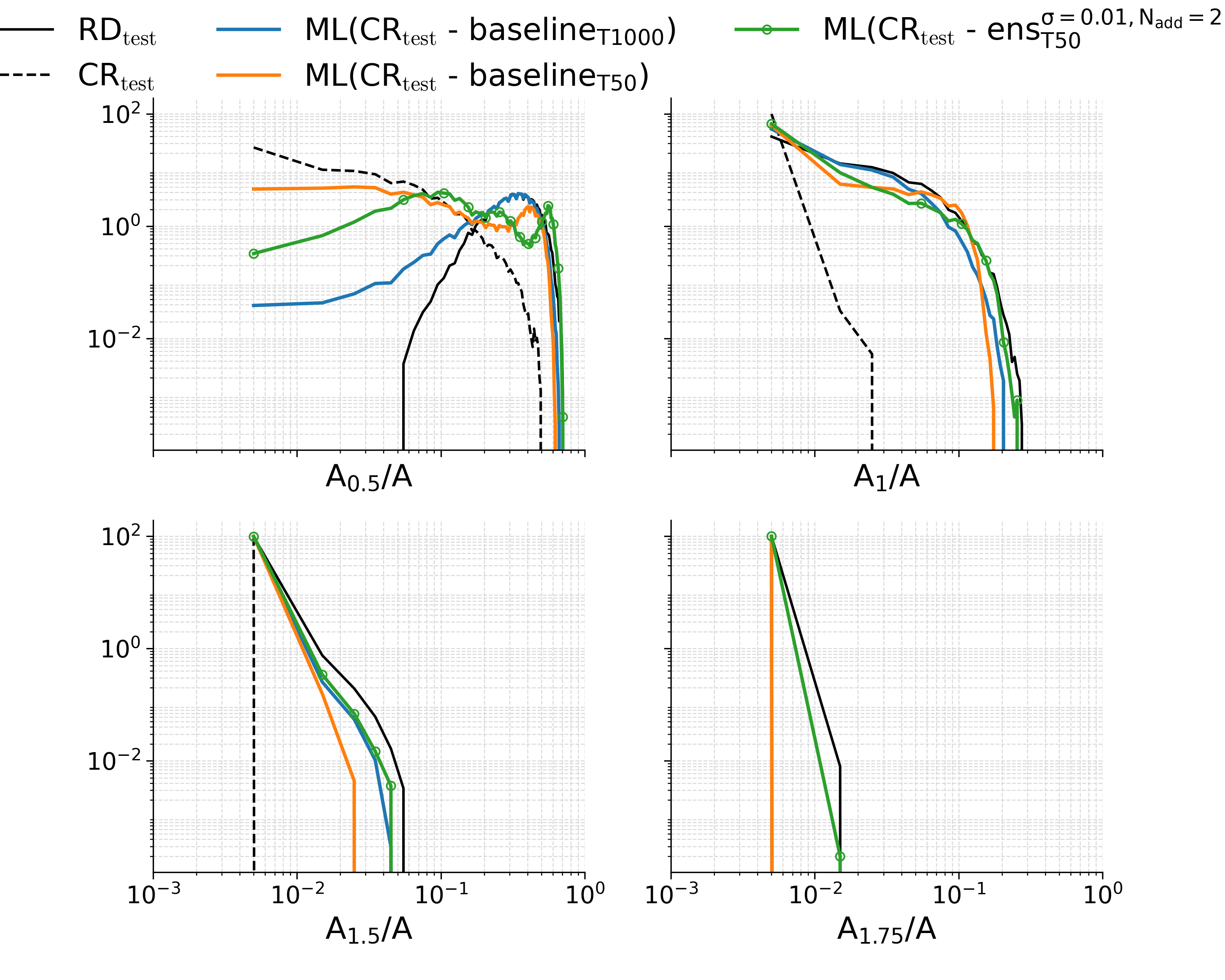}
        \caption{Threshold exceedance: $\psi_2$}
        \label{fig:psi2}
    \end{subfigure}
\caption{Threshold-exceedance area distributions across increasing amplitude thresholds. Each curve shows the distribution of \(\mathrm{AOT}_c(t)=A_c(t)/A\), the fraction of the domain exceeding threshold a \(c\). Black is the projected high-resolution ground truth, dashed black is the uncorrected imperfect model trajectory, blue is STORN$^*$ trained with \(T_{\mathrm{train}}=1000\), orange is STORN trained with configuration \(T_{\mathrm{train}}=50\), and green is our FiLM-STORN trained with \(T_{\mathrm{train}}=50\), \(\sigma=0.01\), and \(M=2\). The separation at the highest thresholds shows the rare-event regime: the short-data STORN baseline fails, whereas ensemble conditioning recovers the ground truth tail more accurately.}
    \label{fig:aots}
    \vspace{-8pt}
\end{figure}
\vspace{-0.2cm}

%% file: 6_limitations.tex
\section{Limitations} \label{sec:limitations}
\vspace{-0.2cm}
The results should be interpreted in terms of both scope and structural assumptions. The ensemble-based instability proxy is most informative in regimes where the imperfect model exhibits strong dynamical sensitivity relative to the ground truth, particularly when extreme events induce significant deviations between model and ground truth attractors. When model and ground truth statistics are already closely aligned, the additional conditioning signal is weak and the method offers limited improvement over standard coarse or model bias-correction approaches. Conversely, in severely data-limited cases the method can improve tail behavior while leaving bulk statistics under-corrected, simply because there are too few samples for the model to learn both the mean behavior and the far tails. 

Second, the instability proxy is a finite-horizon, small-noise result and regime-dependent. The theory does not cover shocks, discontinuous switching, or general hybrid dynamics. The covariance–instability correspondence relies on small-variance fluctuations, local linearization, and a spectral gap assumption. If noise is too weak, perturbations collapse and fail to resolve anisotropy; if too strong, covariance mixes dynamical and noise-induced variability, degrading alignment with the unstable subspace.


Third, the ensemble needed during training introduces additional computational cost. While the imperfect model is cheaper than the high-resolution ground truth, evolving an ensemble SDE and computing summary statistics is more expensive than a single trajectory. In the quasi-geostrophic experiments, small ensembles and low-order statistics (e.g., means and variances) suffice to yield substantial improvements in tail behavior. Nevertheless, when the imperfect model solver is itself costly or computational budgets are tight, this overhead must be weighed against the reduction in required high-resolution data.

Fourth, the empirical evaluation is intentionally scoped. We consider a prototypical quasi-geostrophic system with a probabilistic recurrent backbone, alongside a low-dimensional toy problem with analytically tractable instability structure.
Accordingly, we do not claim uniform superiority, but demonstrate consistent gains within the studied model–system regimes.

\vspace{-8pt}

%% file: 7_conclusion.tex
\section{Conclusion}\label{sec:conclusion}
\vspace{-0.2cm}

We introduced mechanism-aware ensemble conditioning, a non-intrusive way to convert nudged ensembles of the imperfect model into local instability features to inform data-limited surrogate modeling. The central contribution is the coupling of a dynamical sensor with an architectural interface: a nudged ensemble exposes local finite-time instability geometry, and FiLM provides a lightweight way to inject this geometry into otherwise retained correctors. 
This shifts the emphasis from proposing a new architecture to changing what information the architecture is conditioned on, which is important for data efficiency: the gain comes from exploiting additional structure in the dynamics, rather than from increasing model capacity. 

We demonstrated this approach in two distinct backbone families, a probabilistic recurrent FiLM-STORN model for high-dimensional QG and a deterministic FiLM residual-attention model for the low-dimensional controlled system. The QG experiments show improved rare-event statistics in the data-limited regime, while the low-dimensional OTD comparison and correction-pipeline support the proposed geometric mechanism. These results suggest that ensemble-derived instability geometry is a reusable conditioning signal for rare-event emulation, rather than a feature tied to a single surrogate architecture.

%% file: appendix.tex
\section{Proofs of main results}\label{app:proofs}
\input{appendix_sections/proofs}

\section{Low-Dimensional System: methodology and implementation details} \label{app:lowdimsystem}

\input{appendix_sections/lowdim}

\section{QG model: governing equations and numerical implementation}\label{app:QG}

\input{appendix_sections/QG_model}

\section{QG model: methodology and implementation details}\label{app:ML_QG}
\input{appendix_sections/ML}

\section{QG model: evaluation and additional results}\label{app:evaluation}
\input{appendix_sections/add_results}

%% file: appendix_sections/proofs.tex
n this section, we provide the proofs of our main results.

\subsection{Well-posedness and moment control for the nudged ensemble}\label{app:wellposedness}

We first establish that the nudged ensemble SDE has a unique solution with finite moments on a finite time interval. This standard result provides the foundation for the covariance analysis below \citep{oksendal2003}. The relevance of the ensemble lies not in stochasticity per se, but in the geometry of the resulting covariance.

\begin{lemma}[Well-posedness and moment bounds] \label{lem:wellposedness}
Let \(B_t\) be an \(n\)-dimensional Brownian motion on a filtered probability space \((\Omega,\mathcal F,(\mathcal F_t)_{t\ge0},\mathbb P)\) satisfying the usual conditions. Assume $T>0$, and assume that $f \in C^2(\mathbb{R}^n; \mathbb{R}^n)$ is globally Lipschitz with linear growth, i.e.,
\[
\|f(x)-f(y)\| \le K\|x-y\|, \qquad \|f(x)\| \le K(1+\|x\|),
\]
for some $K>0$. Further assume that \(u\in C([0,T];\mathbb R^n)\) is prescribed and that \(t\mapsto \mathcal PF(\mathbf u(t))\) is a prescribed, bounded, measurable function on \([0,T]\). Define the drift
\[
v_t(Q) := f(u(t)+Q) - \mathcal{P}F(\textbf{u}(t)) - \tau^{-1}Q.
\]
Then for every $\tau,\beta>0$ and every \(\mathcal F_0\)-measurable initial condition \(Q(0)\) satisfying \(\mathbb E\|Q(0)\|^2<\infty\), the SDE
\[
dQ(t) = v_t(Q(t))\,dt + \sqrt{\tfrac{2}{\beta}}\,dB_t, \qquad t\in[0,T],
\]
admits a unique strong solution $Q \in C([0,T];\mathbb{R}^n)$. Moreover, for any fixed $p \ge 2$, if $Q(0) \in L^p$, then there exists a constant $C_{p,n}>0$ such that
\[
\sup_{t\in[0,T]} \mathbb{E}\|Q(t)\|^p \le C_{p,n}\Bigl(1 + \mathbb{E}\|Q(0)\|^p\Bigr),
\]
where $C_{p,n}$ depends only on $p$, $n$, $K$, $\tau$, $\beta$, $T$ and the bounds on $\sup_{t\leq T}\| u(t)\| $ and $\sup_{t\leq T}\|\mathcal{P}F(\textbf{u}(t))\| $.
\end{lemma}

\begin{proof}\hfill \newline
\textbf{Existence and uniqueness.}
Since $f$ is globally Lipschitz and $u(t)$ is bounded, the drift $v_t(Q)$ is globally Lipschitz in $Q$, uniformly in $t$, and satisfies a linear growth bound of the form
\[
\|v_t(Q)\| \le C(1+\|Q\|),
\]
for some $C>0$. The diffusion coefficient is constant. Standard results for SDEs with globally Lipschitz coefficients (e.g.\ \cite[Theorem 5.2.1]{oksendal2003}) therefore yield existence and uniqueness of a strong solution.

\vspace{0.5em}
\noindent
\textbf{Moment bounds.}
Fix \(p\ge2\). Applying Itô's formula to \(\phi(x)=\|x\|^p\), with localization if needed, gives
\[
d\|Q(t)\|^p = p\|Q\|^{p-2} Q^\top v_t(Q)\,dt + \frac{p(n+p-2)}{\beta}\|Q\|^{p-2}\,dt + p\|Q\|^{p-2} Q^\top \sqrt{\tfrac{2}{\beta}}\,dB_t.
\]
Taking expectations and using the growth bound on $v_t$ gives
\[
\frac{d}{dt}\mathbb{E}\|Q(t)\|^p \le C\,\mathbb{E}\big[\|Q\|^{p-1}(1+\|Q\|)\big]+C\,\mathbb{E}\|Q\|^{p-2}.
\]
Using Young’s inequality,
\(
\|Q\|^{p-1} \le 1 + \|Q\|^p, \; \|Q\|^{p-2} \le 1 + \|Q\|^p,
\)
we obtain
\[
\frac{d}{dt}\mathbb{E}\|Q(t)\|^p \le C_{p,n} \bigl(1 + \mathbb{E}\|Q(t)\|^p\bigr).
\]
An application of Grönwall’s inequality yields
\[
\sup_{t\in[0,T]} \mathbb{E}\|Q(t)\|^p \le C_{p,n}\Bigl(1 + \mathbb{E}\|Q(0)\|^p\Bigr),
\]
as claimed.
\end{proof}

\subsection{Proof of Theorem \ref{prop:covariance-instability} (Covariance-instability correspondence)}\label{app:covariance-instability}

\begin{theorem}[Covariance--instability correspondence, Theorem \ref{prop:covariance-instability} ]
Assume the conditions of Lemma~\ref{lem:wellposedness}. In addition, assume that
\[
\sup_{x\in\mathbb{R}^n}\|D^2 f(x)\|_{\mathrm{op}}\le H<\infty,
\]
and that the centered fluctuations \(Q'(t)=Q(t)-\bar Q(t)\), where \(\bar Q(t)=\mathbb{E}[Q(t)]\), satisfy
\[
\sup_{t\in[0,T]}\mathbb{E}\|Q'(t)\|^2\le C_2\beta^{-1},
\qquad
\sup_{t\in[0,T]}\mathbb{E}\|Q'(t)\|^4\le C_4\beta^{-2}.
\]
Then, for all $t\in[0,T]$, the following hold:

\begin{enumerate}
\item[(i)] (\textbf{Mean-field consistency}) The ensemble mean satisfies
\[
\dot{\bar{Q}}(t) =f\big(u(t)+\bar{Q}(t)\big) - \mathcal{P}F\big(\textbf{u}(t)\big) - \tau^{-1}\bar{Q}(t) + O(\beta^{-1}).
\]

\item[(ii)] (\textbf{Covariance dynamics}) The covariance
\(\Sigma(t) = \mathbb{E}\big[Q'(t)Q'(t)^\top\big]\)
satisfies the perturbed Lyapunov equation
\begin{equation}
\label{eq:app:covariance_dynamics}
\dot\Sigma(t)=L_\tau(t)\Sigma(t)+\Sigma(t)L_\tau(t)^\top + \frac{2}{\beta}I +R_\Sigma(t),
\qquad
\|R_\Sigma(t)\|_{\mathrm{op}}\le C_T\beta^{-3/2}.
\end{equation}

where $L_\tau(t) = Df(u(t)+\bar{Q}(t)) - \tau^{-1}I$.

\item[(iii)] (\textbf{Deformation-kernel representation}) Consequently, if \(\Phi(t,s)\) is the fundamental solution generated by \(L_\tau(t)\), then
\[
\Sigma(t)
=
\Phi(t,0)\Sigma(0)\Phi(t,0)^\top
+
\frac{2}{\beta}\int_0^t \Phi(t,s)\Phi(t,s)^\top\,ds
+
O_T(\beta^{-3/2}).
\]
Equivalently,
\[
\beta\Sigma(t)
=
\Phi(t,0)\beta\Sigma(0)\Phi(t,0)^\top
+
2\int_0^t \Phi(t,s)\Phi(t,s)^\top\,ds
+
O_T(\beta^{-1/2}).
\]
The representation separates the propagated initial covariance from the accumulated noise-forced deformation. When the initial contribution is negligible relative to the deformation integral, the latter determines the leading covariance geometry.
\end{enumerate}

\end{theorem}

\begin{proof}
All differential identities below hold for a.e. \(t\in[0,T]\); the integral representation in part~(iii) holds for every \(t\in[0,T]\).

\textbf{(i) Mean-field consistency.}
Decompose \(Q(t)=\bar Q(t)+Q'(t),\ \bar Q(t)=\mathbb E[Q(t)],\ \mathbb E[Q'(t)]=0.\) into its mean $\bar Q$ and centered fluctuations $Q'$
For notational simplicity define the time dependent drift term
\[
v_t(Q):=f(u(t)+Q)-\mathcal P F(\mathbf u(t))-\tau^{-1}Q
\]
The linearized nudged operator is
\[
L_\tau(t):=Df(u(t)+\bar Q(t))-\tau^{-1}I
\]
and define the Taylor remainder
\[
r_t(Q'):=f(u(t)+\bar Q(t)+Q')-f(u(t)+\bar Q(t))-Df(u(t)+\bar Q(t))Q' .
\]
Since \(D^2f\) is uniformly bounded by assumption, with
\[
\|D^2f(x)\|_{\mathrm{op}}
:=
\sup_{\|a\|=\|b\|=1}\|D^2f(x)[a,b]\|
\le H,
\]
Taylor's theorem gives \(\|r_t(Q')\|\le \frac{H}{2}\|Q'\|^2 \). Thus
\[
v_t(\bar Q+Q')
=
v_t(\bar Q)+L_\tau(t)Q'+r_t(Q').
\]
Taking expectations in the SDE gives
\[
\dot{\bar Q}(t)=\mathbb E[v_t(Q(t))]
=
v_t(\bar Q(t))+\mathbb E[r_t(Q'(t))],
\]
where we used that the fluctuations are mean-zero, i.e., \(\mathbb E[Q'(t)]=0\). Therefore
\[
\|\mathbb E[r_t(Q'(t))]\|
\le
\frac{H}{2}\mathbb E\|Q'(t)\|^2
\le
\frac{H C_2}{2}\beta^{-1}.
\]
Hence
\[
\dot{\bar Q}(t)
=
f(u(t)+\bar Q(t))-\mathcal P F(\mathbf u(t))-\tau^{-1}\bar Q(t)
+
O(\beta^{-1}),
\]
uniformly on \([0,T]\).

\noindent
\textbf{(ii) Covariance dynamics.}
Subtracting the mean equation from the SDE yields the centered fluctuation equation
\[
dQ'(t)
=
\Big[
L_\tau(t)Q'(t)
+
r_t(Q'(t))
-
\mathbb E[r_t(Q'(t))]
\Big]dt
+
\sqrt{\frac{2}{\beta}}\,dB_t .
\]
Centering leaves the diffusion term unchanged. Its quadratic variation produces the covariance source \((2/\beta)I\). Apply Itô's product rule to \(Q'(t)Q'(t)^\top\). Since the quadratic variation of
\(\sqrt{2/\beta}\,B_t\) is \((2/\beta)I\,dt\), we obtain
\[
\begin{aligned}
\frac{d}{dt}\Sigma(t)
&=
L_\tau(t)\Sigma(t)+\Sigma(t)L_\tau(t)^\top
+
\frac{2}{\beta}I
+
R_\Sigma(t),
\end{aligned}
\]
where
\[
\begin{aligned}
R_\Sigma(t)
&=
\mathbb E\Big[
\big(r_t(Q')-\mathbb E r_t(Q')\big)Q'^\top
\Big]
+
\mathbb E\Big[
Q'\big(r_t(Q')-\mathbb E r_t(Q')\big)^\top
\Big].
\end{aligned}
\]
Again, using \(\mathbb E[Q']=0\) such that the terms involving \(\mathbb E r_t(Q')\) vanish, and hence
\[
R_\Sigma(t)
=
\mathbb E\big[r_t(Q')Q'^\top\big]
+
\mathbb E\big[Q'r_t(Q')^\top\big].
\]
Consequently,
\[
\begin{aligned}
\|R_\Sigma(t)\|_{\mathrm{op}}
&\le
2\,\mathbb E\big[\|r_t(Q')\|\,\|Q'\|\big] \\
&\le
H\,\mathbb E\|Q'(t)\|^3 \\
&\le
H\big(\mathbb E\|Q'(t)\|^2\big)^{1/2}
  \big(\mathbb E\|Q'(t)\|^4\big)^{1/2} \\
&\le
H(C_2C_4)^{1/2}\beta^{-3/2}.
\end{aligned}
\]
Thus \(\|R_\Sigma(t)\|_{\mathrm{op}}\le C_T\beta^{-3/2}\), uniformly on \([0,T]\).

\vspace{0.8em}
\noindent
\textbf{(iii) Deformation-kernel representation.}
Let \(\Phi(t,s)\) be the fundamental solution generated by \(L_\tau(t)\) i.e., \(\partial_t\Phi(t,s)=L_\tau(t)\Phi(t,s),\;\Phi(s,s)=I.\)
Since \(f\) is globally Lipschitz and \(C^1\), \(\|Df\|_{\mathrm{op}}\le K\), and hence
\[
\|L_\tau(t)\|_{\mathrm{op}}\le K+\tau^{-1}.
\]
Therefore
\[
\sup_{0\le s\le t\le T}\|\Phi(t,s)\|_{\mathrm{op}}
\le
\exp\big((K+\tau^{-1})T\big)
=:C_\Phi(T).
\]

Fix \(t\in[0,T]\) and define
\[
\widetilde\Sigma(s;t)
:=
\Phi(t,s)\Sigma(s)\Phi(t,s)^\top .
\]
Using \(\partial_s\Phi(t,s)=-\Phi(t,s)L_\tau(s)\), we get
\[
\begin{aligned}
\frac{d}{ds}\widetilde\Sigma(s;t)
&=
\Phi(t,s)
\Big[
\dot\Sigma(s)
-
L_\tau(s)\Sigma(s)
-
\Sigma(s)L_\tau(s)^\top
\Big]
\Phi(t,s)^\top \\
&=
\Phi(t,s)
\Big[
\frac{2}{\beta}I+R_\Sigma(s)
\Big]
\Phi(t,s)^\top .
\end{aligned}
\]
Integrating from \(s=0\) to \(s=t\) gives
\[
\Sigma(t)
-
\Phi(t,0)\Sigma(0)\Phi(t,0)^\top
=
\frac{2}{\beta}\int_0^t
\Phi(t,s)\Phi(t,s)^\top\,ds
+
\int_0^t
\Phi(t,s)R_\Sigma(s)\Phi(t,s)^\top\,ds .
\]
The last term satisfies
\[
\begin{aligned}
\left\|
\int_0^t
\Phi(t,s)R_\Sigma(s)\Phi(t,s)^\top\,ds
\right\|_{\mathrm{op}}
&\le
C_\Phi(T)^2
\int_0^t \|R_\Sigma(s)\|_{\mathrm{op}}\,ds \\
&\le
C_\Phi(T)^2 T C_T\beta^{-3/2}
=
O_T(\beta^{-3/2}).
\end{aligned}
\]
Therefore
\[
\Sigma(t)
=
\Phi(t,0)\Sigma(0)\Phi(t,0)^\top
+
\frac{2}{\beta}\int_0^t
\Phi(t,s)\Phi(t,s)^\top\,ds
+
O_T(\beta^{-3/2})
\]
in operator norm or equivalently, multiplying by \(\beta\) gives
\[
\beta\Sigma(t)
=
\Phi(t,0)\beta\Sigma(0)\Phi(t,0)^\top
+
2\int_0^t
\Phi(t,s)\Phi(t,s)^\top\,ds
+
O_T(\beta^{-1/2}).
\]
Thus, at the \(O(\beta^{-1})\) covariance scale, the covariance geometry is generated by the finite-time deformation kernels of the linearized nudged dynamics.
\end{proof}

\paragraph{Scope of the covariance correspondence.}
Theorem~\ref{prop:covariance-instability} is a finite-horizon, small-noise result. Smoothness and the bound on \(D^2f\) control the local-linearization remainder. The second- and fourth-moment bounds are additional regime assumptions, they do not follow from Lemma~\ref{lem:wellposedness} alone. Interpreting the covariance primarily as a noise-forced deformation probe also requires that the propagated initial covariance not dominate the deformation integral. The theorem does not cover shocks, discontinuous switching, or general hybrid dynamics, and it does not by itself guarantee accurate extreme-event probabilities.

\paragraph{Slow variability.}
Let \(S\) denote a resolved slow regime and \(E\) an extreme event. Then
\[
\mathbb P(E)=\int \mathbb P(E\mid S=s)\,\pi(\mathrm ds),
\]
where \(\pi\) is the long-run distribution of slow regimes. The ensemble probes amplification conditional on the trajectory being sampled. Resolved slow variability enters through the base state and the linearized operator \(L_\tau(t)\), the theorem does not require a single timescale. However, if a relevant slow variable is absent from the coarse state or a regime is missing from the reference window, improved conditional extreme-event statistics need not imply accurate regime occupancies or transition rates.

\subsection{Proof of Theorem \ref{theo:finite_ensemble_approx} (Finite-ensemble approximation of instability subspaces)}\label{app:finite_ensemble_approx}

\begin{theorem}[Finite-sample approximation of covariance-induced instability subspaces]
\ref{theo:finite_ensemble_approx}
Fix \(t\in[0,T]\). Assume that the centered ensemble fluctuations \(Q^{\prime(1)}(t),\ldots,Q^{\prime(M)}(t)\) are iid satisfying \(\mathbb E[Q'(t)]=0\) and \(\mathbb E[\|Q'(t)\|_2^4]\le \kappa<\infty\). Let
\[
\Sigma(t)=\mathbb E[Q'(t)Q'(t)^\top],\qquad \widehat\Sigma_M(t)=\frac1M\sum_{m=1}^M Q^{\prime(m)}(t)Q^{\prime(m)}(t)^\top .
\]
Let \(R(t)\) and \(\widehat R(t)\) denote the orthogonal projectors onto the leading \(r\)-dimensional eigenspaces of \(\Sigma(t)\)  and \(\widehat\Sigma_M(t)\), respectively. Here, \(1\le r<n\), and eigenvalues are ordered from largest to smallest. Assume a spectral gap holds uniformly, i.e. \(\lambda_r(\Sigma(t))-\lambda_{r+1}(\Sigma(t))\ge \delta>0\). Then
\[
\mathbb E[\|\widehat R(t)-R(t)\|_{\mathrm{op}}]\le \frac{2\sqrt{\kappa}}{\delta\sqrt M}.
\]
Moreover, for every \(\varepsilon>0\),
\[
\mathbb P\{\|\widehat R(t)-R(t)\|_{\mathrm{op}}\ge\varepsilon\}\le \frac{4\kappa}{\delta^2M\varepsilon^2}.
\]
\end{theorem}

\begin{proof}
Define \(X_m=Q^{\prime(m)}(t)Q^{\prime(m)}(t)^\top-\Sigma(t)\) and \(E_M=\widehat\Sigma_M(t)-\Sigma(t)=M^{-1}\sum_{m=1}^M X_m\). Then \(\mathbb E[X_m]=0\), and since the \(X_m\)'s are independent mean-zero random matrices in the Hilbert space of matrices equipped with the Frobenius inner product,
\[
\mathbb E[\|E_M\|_F^2]=\frac1{M^2}\sum_{m=1}^M\mathbb E[\|X_m\|_F^2]=\frac1M\mathbb E[\|X_1\|_F^2].
\]
Now
\[
\mathbb E[\|X_1\|_F^2]=\mathbb E[\|Q'(t)Q'(t)^\top-\Sigma(t)\|_F^2]=\mathbb E[\|Q'(t)Q'(t)^\top\|_F^2]-\|\Sigma(t)\|_F^2\le \mathbb E[\|Q'(t)\|_2^4]\le \kappa,
\]
where we used \(\mathbb E[Q'(t)Q'(t)^\top]=\Sigma(t)\) and \(\|Q'(t)Q'(t)^\top\|_F^2=\|Q'(t)\|_2^4\). Hence
\[
\mathbb E[\|E_M\|_{\mathrm{op}}]\le \mathbb E[\|E_M\|_F]\le \{\mathbb E[\|E_M\|_F^2]\}^{1/2}\le \sqrt{\kappa/M}.
\]
By Davis--Kahan,
\[
\|\widehat R(t)-R(t)\|_{\mathrm{op}}\le \frac{2}{\delta}\|\widehat\Sigma_M(t)-\Sigma(t)\|_{\mathrm{op}}=\frac{2}{\delta}\|E_M\|_{\mathrm{op}}.
\]
Taking expectations gives
\[
\mathbb E[\|\widehat R(t)-R(t)\|_{\mathrm{op}}]\le \frac{2}{\delta}\mathbb E[\|E_M\|_{\mathrm{op}}]\le \frac{2\sqrt{\kappa}}{\delta\sqrt M}.
\]
For the probability bound, Markov's inequality gives
\[
\mathbb P\{\|E_M\|_{\mathrm{op}}\ge a\}\le \mathbb P\{\|E_M\|_F\ge a\}\le \frac{\mathbb E[\|E_M\|_F^2]}{a^2}\le \frac{\kappa}{Ma^2}.
\]
Using Davis--Kahan again, \(\|\widehat R(t)-R(t)\|_{\mathrm{op}}\ge\varepsilon\) implies \(\|E_M\|_{\mathrm{op}}\ge \delta\varepsilon/2\). Therefore
\[
\mathbb P\{\|\widehat R(t)-R(t)\|_{\mathrm{op}}\ge\varepsilon\}\le \frac{4\kappa}{\delta^2M\varepsilon^2}.
\]
\end{proof}
\begin{remark}[Sub-Gaussian strengthening]
If \(Q'(t)\) is sub-Gaussian, then standard covariance concentration yields exponential high-probability bounds for \(\|\widehat\Sigma_M(t)-\Sigma(t)\|_{\mathrm{op}}\), and Davis--Kahan \cite{davis1970} gives the corresponding exponential projector bound. The fourth-moment theorem above is the assumption-minimal statement needed for the \(M^{-1/2}\) mean projector error used in this work.
\end{remark}

%% file: appendix_sections/lowdim.tex
\subsection{Dynamical system and imperfect model.}
\label{app:lowdim_system} 
To validate the covariance-based instability proxy in a setting where the unstable directions are known independently from tangent dynamics, we consider the three-dimensional nonlinear system of \cite{babaee2016}. This example was chosen because it exhibits alternating regimes of non-normal transient growth and exponential growth/decay, so that the dominant finite-time instability direction changes repeatedly along a single trajectory and cannot be captured by a fixed eigenvector of the linearized operator. Let \(z=(z_1,z_2,z_3)^\top\). The governing equations are
\begin{equation}
\dot{z}_1 = -a_1 z_1 + \epsilon z_2 + b z_3,\qquad
\dot{z}_2 = \epsilon^{-1} z_1 - a_2 z_2,\qquad
\dot{z}_3 = b z_3\left(\frac{1}{\sqrt{z_1^2+z_2^2}}-1\right),
\label{eq:lowdim_system}
\end{equation}
with parameters
\begin{equation}
a_1=a_2=2,\qquad \epsilon=0.05,\qquad b=20,
\end{equation}
as in \cite{babaee2016}. For these values, the dynamics are nearly periodic and each cycle contains four distinct regimes: non-normal growth in the \(z_1\)--\(z_2\) plane, decay toward the origin in the same plane, growth in the \(z_3\) direction, and decay in \(z_3\). The singularity at \(z_1^2+z_2^2=0\) further induces strong finite-time amplification near the origin, making the example a stringent test of whether an ensemble covariance can recover the relevant instability subspace. The imperfect and biased model we introduce is the same nonlinear three-dimensional system, but with an additional linear damping term \(-\gamma z_3\) in the third equation to moderate the strongest growth near the singular region.

We compare the dominant OTD mode \(q_{\mathrm{OTD}}(t)\) from \cite{babaee2016,babaee2017} with the leading eigenvector \(q_{\mathrm{ens}}(t)\) of the empirical covariance of the centered ensemble generated around the nudged trajectory. If \(z^{(m)}(t)\), \(m=1,\dots,M\), denote the ensemble realizations and \(\bar z(t)=M^{-1}\sum_{m=1}^M z^{(m)}(t)\), then
\begin{equation}
\hat\Sigma_M(t)=\frac{1}{M}\sum_{m=1}^M \big(z^{(m)}(t)-\bar z(t)\big)\big(z^{(m)}(t)-\bar z(t)\big)^\top,
\end{equation}
and \(q_{\mathrm{ens}}(t)\) is taken as the leading eigenvector of \(\hat\Sigma_M(t)\). The nudging timescale is fixed to \(\tau=8\). The purpose of this example is not predictive accuracy, but mechanistic validation: it shows that the ensemble covariance tracks the same time-dependent unstable subspace identified by OTD theory without integrating tangent dynamics or forming Jacobians explicitly.

\begin{figure}[H]
        \centering
        \includegraphics[width=1\linewidth]{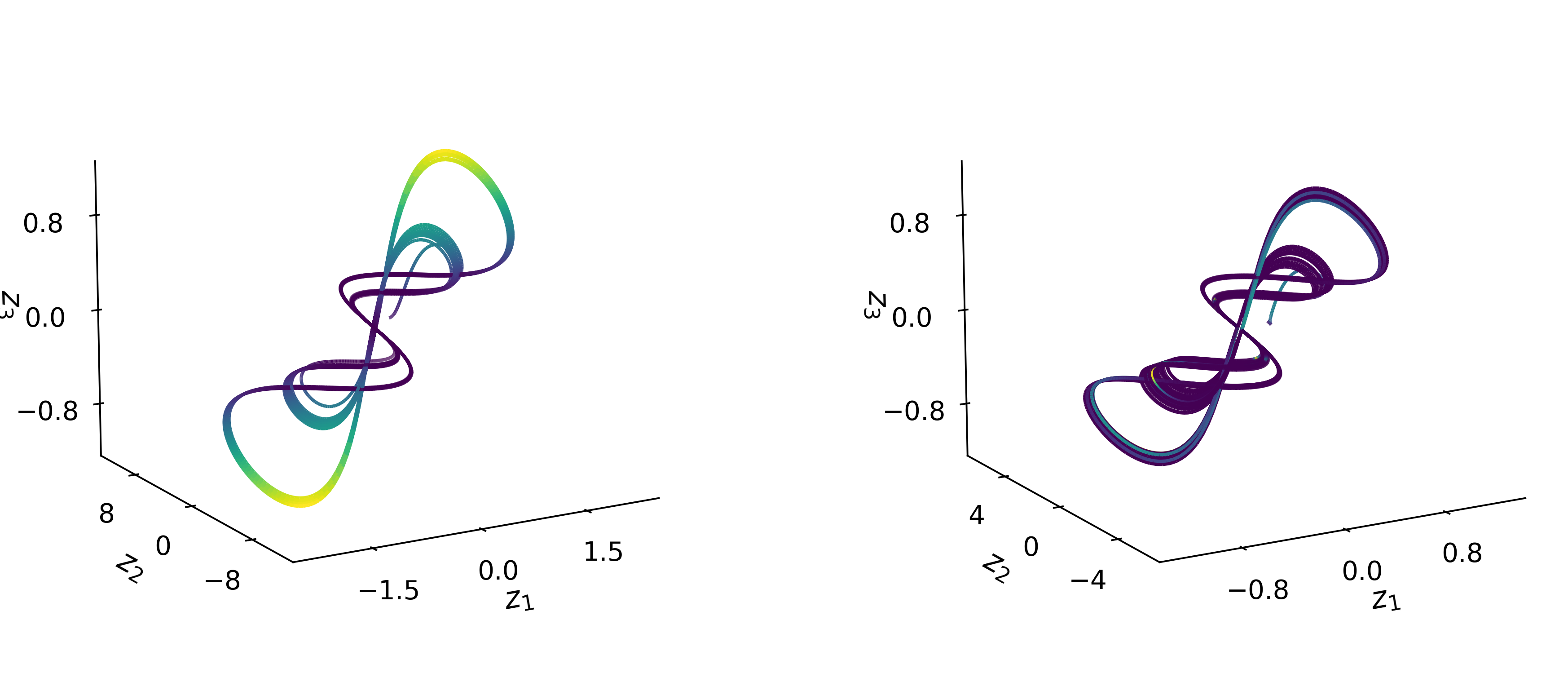}
        \caption{Visualization of the 3D dynamical system with slow–fast dynamics. Left: reference system; right: imperfect, damped system. Color corresponds to the $z_3$ axis magnitude.}
        \label{fig:lowdim_otd}
\end{figure}

\subsection{Training and test-time learning problem.}
\label{app:lowdim_learning_problem}

The correction task is supervised sequence-to-sequence regression. From a biased model trajectory \(z^{\rm c}(t)\) and reference trajectory \(z^{\rm ref}(t)\), we construct overlapping windows
\[
    X_k =
    \bigl(
        z^{\rm c}(t_k), z^{\rm c}(t_k+\Delta t_{\rm ML}),
        \ldots,
        z^{\rm c}(t_k+(H-1)\Delta t_{\rm ML})
    \bigr),
\]
\[
    Y_k =
    \bigl(
        z^{\rm ref}(t_k), z^{\rm ref}(t_k+\Delta t_{\rm ML}),
        \ldots,
        z^{\rm ref}(t_k+(H-1)\Delta t_{\rm ML})
    \bigr).
\]
For FiLM-conditioned models, the corresponding context window is
\[
    C_k =
    \bigl(
        c_{t_k}, c_{t_k+\Delta t_{\rm ML}},
        \ldots,
        c_{t_k+(H-1)\Delta t_{\rm ML}}
    \bigr).
\]
The model is trained on horizons of $T_{\rm train}\in\{5,50,150\}$ and evaluated on an independent long trajectory of length \( T_{\rm test}=500\).

At test time, the model is non-intrusive and for ensemble-conditioned models, the default deployment context is therefore computed from the imperfect model trajectory alone using a causal rolling window. This ensures that the deployed correction uses only the model trajectory and no ground truth leakage.

\subsection{Residual-attention corrector.}
\label{app:lowdim_architecture}

The low-dimensional corrector is a residual attention model. Given an input state \(z^{\rm c}_t\in\mathbb{R}^3\), an input projection maps it to a hidden token
\[
    h_t^{(0)} = W_{\rm in} z^{\rm c}_t + b_{\rm in}
    \in\mathbb{R}^{d_{\rm hid}},
\]
followed by fixed sinusoidal positional encoding. 
The sequence is then processed by a Transformer encoder with \(n_{\rm layers}\) self-attention blocks. For layer \(\ell=1,\ldots,n_{\rm layers}\),
\[
    \widehat h_t^{(\ell)}
    =
    h_t^{(\ell-1)}
    +
    \operatorname{Dropout}
    \left[
        \operatorname{MHA}
        \left(
            \operatorname{LayerNorm}(h^{(\ell-1)})
        \right)_t
    \right],
\]
\[
    h_t^{(\ell)}
    =
    \widehat h_t^{(\ell)}
    +
    \operatorname{Dropout}
    \left[
        \operatorname{FFN}
        \left(
            \operatorname{LayerNorm}(\widehat h_t^{(\ell)})
        \right)
    \right],
\]
where
\[
    \operatorname{FFN}(h)
    =
    W_2\,\operatorname{Dropout}
    \left(
        \operatorname{GELU}(W_1h+b_1)
    \right)+b_2 .
\]
Here, MHA denotes multi-head self-attention and FFN denotes the feed-forward network. The final hidden state is \(h_t=h_t^{(n_{\rm layers})}\).

The model predicts a residual correction and a gate,
\[
    \delta_t = W_\delta h_t+b_\delta\in\mathbb{R}^3,
    \qquad
    g_t = \operatorname{sigmoid}(W_g h_t+b_g)\in(0,1)^3,
\]
and outputs
\[
    \widehat z_t
    =
     z^{\rm c}_t + g_t\odot \delta_t .
\]
The gate bias is initialized at \(b_g=-2.0\),reducing the initial contribution of the residual correction.


For ensemble-conditioned models, the attention backbone is unchanged. A small context network maps \(c_t\) to FiLM parameters:
\[
    h^{\rm ctx}_t
    =
    \tanh(W_{{\rm ctx},1}c_t+b_{{\rm ctx},1})
    \in\mathbb{R}^{48},
\]
\[
    \gamma_t=W_\gamma h^{\rm ctx}_t+b_\gamma,
    \qquad
    \beta_t=W_\beta h^{\rm ctx}_t+b_\beta ,
\]
with \(\gamma_t,\beta_t\in\mathbb{R}^{d_{\rm hid}}\). FiLM modulation is applied after the attention encoder and before the residual output heads:
\[
    h_t^{\rm FiLM}
    =
    (1+\gamma_t)\odot h_t+\beta_t .
\]
The residual correction and gate are then computed from \(h_t^{\rm FiLM}\) instead of \(h_t\). This keeps the backbone, loss, optimizer, and output parameterization fixed, isolating the contribution of the ensemble-derived conditioning signal.

\begin{algorithm}[t]
\caption{Residual-attention corrector training for the low-dimensional system}
\label{alg:lowdim_attention_train}
\begin{algorithmic}[1]
\Require Imperfect model trajectory \(z^{\rm c}(t)\), ground truth trajectory \(z^{\rm ref}(t)\),
window length \(H\), stride \(s\), batch size \(B\)
\Ensure trained residual-attention corrector \(G_\theta\)
\State normalize \(z^{\rm c}\) and \(z^{\rm ref}\) using training-ground truth statistics
\State extract overlapping windows \((X_k,Y_k)\) from \((z^{\rm c},z^{\rm ref})\)
\For{epoch \(=1,\ldots,N_{\rm epoch}\)}
    \For{batches of \(B\) windows}
        \State compute \(\widehat Y = G_\theta(X)\)
        \State evaluate \(L_{\rm path}+\lambda_\delta L_\delta
        +\lambda_{\rm mom}L_{\rm mom}+\lambda_{\rm tail}L_{\rm tail}\)
        \State update \(\theta\) with Adam
    \EndFor
\EndFor
\end{algorithmic}
\end{algorithm}

\begin{algorithm}[t]
\caption{FiLM residual-attention training with ensemble-derived context}
\label{alg:lowdim_film_attention_train}
\begin{algorithmic}[1]
\Require Imperfect model trajectory \(z^{\rm c}(t)\), ground truth trajectory \(z^{\rm ref}(t)\),
ensemble members \(\{z^{(m)}(t)\}_{m=1}^M\), context type
\(\mathrm{ctx}\in\{\mathrm{mean/std/eigs},\mathrm{eigvec}\}\)
\Ensure trained FiLM residual-attention corrector \(G_\theta\)
\State compute context sequence \(c_t\) from the ensemble:
\[
c_t=
\begin{cases}
(\mu_t,s_t,\lambda_1(t),\lambda_2(t)), & \mathrm{ctx}=\mathrm{mean/std/eigs},\\
(v_1(t),\ldots,v_r(t)), & \mathrm{ctx}=\mathrm{eigvec}.
\end{cases}
\]
\State extract overlapping windows \((X_k,Y_k,C_k)\) from
\((z^{\rm c},z^{\rm ref},c)\)
\State initialize the same residual-attention backbone as in Algorithm~\ref{alg:lowdim_attention_train}
\For{epoch \(=1,\ldots,N_{\rm epoch}\)}
    \For{batches of \(B\) windows}
        \State compute FiLM parameters \((\gamma_t,\beta_t)=g_\theta(c_t)\)
        \State modulate attention representation
        \(h_t^{\rm FiLM}=(1+\gamma_t)\odot h_t+\beta_t\)
        \State compute \(\widehat Y=G_\theta(X,C)\)
        \State evaluate the same loss as Algorithm~\ref{alg:lowdim_attention_train}
        \State update \(\theta\) with Adam
    \EndFor
\EndFor
\end{algorithmic}
\end{algorithm}

\begin{algorithm}[t]
\caption{Long-horizon low-dimensional evaluation}
\label{alg:lowdim_eval}
\begin{algorithmic}[1]
\Require trained model \(G_\theta\), free-running model trajectory \(z^{\rm c}(t)\),
test horizon \(T_{\rm test}\), ground truth trajectory \(z^{\rm ref}(t)\) for diagnostics only
\Ensure corrected trajectory \(\widehat z(t)\), long-horizon diagnostics
\State build test windows \(X_{\rm test}\) from the free-running model trajectory
\If{model uses FiLM}
    \State build causal test-time context \(C_{\rm test}\) from rolling imperfect model statistics/eigendirections
    \State compute \(\widehat z(t)=G_\theta(X_{\rm test},C_{\rm test})\)
\Else
    \State compute \(\widehat z(t)=G_\theta(X_{\rm test})\)
\EndIf
\State compute PDFs, \(D_{\rm KL}\), \(L_{\log}\), percentile exceedance frequencies,
and \(\mathcal E_{99}\) against \(z^{\rm ref}(t)\)
\end{algorithmic}
\end{algorithm}

\subsection{Training loss.}
\label{app:lowdim_loss}

The deterministic residual corrector is trained with
\[
    L
    =
    L_{\rm path}
    +
    \lambda_\delta L_\delta
    +
    \lambda_{\rm mom}L_{\rm mom}
    +
    \lambda_{\rm tail}L_{\rm tail}.
\]
The pathwise loss is
\[
    L_{\rm path}
    =
    \mathbb{E}_t
    \left[
        w_{\rm tail}(t)
        \sum_{i=1}^3
        \alpha_i
        \bigl(
            \widehat z_i(t)-z^{\rm ref}_i(t)
        \bigr)^2
    \right],
\]
with coordinate weights \(\alpha=(1.5,1.5,3.0)\).
The third coordinate is weighted more strongly because the strongest extreme growth occurs in the \(z_3\) direction. The tail weight is
\[
    w_{\rm tail}(t)
    =
    \begin{cases}
        1+w_{\rm tail}, & \max_i |z^{\rm ref}_i(t)|>c_{\rm tail},\\
        1, & \text{otherwise},
    \end{cases}
\]
where \(w_{\rm tail}=3.0\) and \(c_{\rm tail}=2.0\) in normalized units.

The residual regularization term is
\[
    L_\delta
    =
    \mathbb{E}_t
    \left[
        \|\widehat z(t)-z^{\rm c}(t)\|_2^2
    \right],
\]
which discourages unnecessary deviations from the imperfect model trajectory. Moment matching is enforced by
\[
    L_{\rm mom}
    =
    \|\mu_{\rm pred}-\mu_{\rm ref}\|_2^2
    +
    \|s_{\rm pred}-s_{\rm ref}\|_2^2,
\]
where the means and standard deviations are computed over each training window and averaged over the batch.

To directly penalize tail-frequency mismatch, we use smooth exceedance indicators over percentile-based thresholds. Let \(\mathcal{C}=\{90,95,99\}\) denote the set of desired percentiles. For coordinate \(i\) and threshold \(c\in\mathcal{C}\), we compute the exceedance indicator as
\[
    I_{\rm smooth}(|z_i|>c_{i,p})
    =
    \operatorname{sigmoid}
    \left(
        \frac{|z_i|-c_p}{T_{\rm tail}}
    \right),
    \qquad
    T_{\rm tail}=0.05 .
\] where $c_p$ denotes the state value corresponding to the $p-$th percentile, rather than the percentile rank itself.
The tail loss is
\[
    L_{\rm tail}=\mathbb{E}\left[\frac{1}{3|\mathcal{C}|}\sum_{i=1}^3       \sum_{c\in\mathcal{C}}\left(P_{\rm pred}(|z_i|>c)- P_{\rm ref}(|z_i|>c)       \right)^2\right].
\]

\subsection{Hardware and runtime.}\label{app:lowdim:runtime}
The low-dimensional proof-of-concept experiment is computationally lightweight, completing in approximately 10 minutes on a single CPU core with negligible memory overhead, under the same CPU-based setup as the high-dimensional experiments.

\subsection{Hyperparameters.}
\label{app:lowdim_hyperparameters}
All hyperparameters for the low-dimensional experiments are summarized in Tables~\ref{tab:lowdim_system_params} and~\ref{tab:lowdim_arch_params}.

\begin{table}[H]
\centering
\small
\caption{
Configuration of the low-dimensional controlled system. The model is deliberately biased through additional damping in the intermittently growing \(z_3\) direction, creating a non-intrusive correction task while preserving the finite-time instability structure of the ground truth dynamics.
}
\label{tab:lowdim_system_params}
\begin{tabular}{ll}
\hline
\textbf{Parameter} & \textbf{Value / Description} \\
\hline

\multicolumn{2}{l}{\textit{Model setup}} \\
\quad State variable & \(z=(z_1,z_2,z_3)^\top\) \\
\quad Imperfect model & ground truth vector field with additional damping \(-\gamma z_3\) in the third equation \\[0.3em]

\multicolumn{2}{l}{\textit{Physical parameters}} \\
\quad \(a_1,a_2, b\) & \(2,2, 20\) \\
\quad \(\epsilon\) & \(0.05\) \\
\quad Imperfect model damping \(\gamma\) & 15 \\[0.3em]

\multicolumn{2}{l}{\textit{Nudging / ensemble}} \\
\quad Nudging timescale \(\tau\) & \(8\) \\
\quad Noise amplitudes \(\sigma\) & \(10^{-6},\,5\times 10^{-6},\,10^{-5}\) \\
\quad Ensemble sizes \(M\) & \(1,2,4,8,16\) \\[0.3em]

\multicolumn{2}{l}{\textit{Training / evaluation horizons}} \\
\quad Training horizons & \(T_{\rm train}\in\{5, 50,150\}\) \\
\quad Test horizon & \(T_{\rm test}=500\) \\
\quad Repeated runs & \(n=5\)\\
\hline
\end{tabular}
\end{table}

\begin{table}[H]
\centering
\small
\vspace{-10pt}
\caption{
Architecture parameters for the low-dimensional residual-attention corrector. The no-FiLM and FiLM models share the same temporal backbone and output parameterization; the FiLM variants add only the ensemble-derived context interface.
}
\label{tab:lowdim_arch_params}
\begin{tabular}{ll}
\hline
\textbf{Parameter} & \textbf{Value / Description} \\
\hline

\multicolumn{2}{l}{\textit{Backbone}} \\
\quad Architecture & residual temporal attention corrector \\
\quad Input state & \(z_t^{\rm c}\in\mathbb R^3\) \\
\quad Hidden width \(d_{\rm hid}\) & \(96\) \\
\quad Attention (heads, layers) & \((4\; , 3)\) \\
\quad Head dimension & \(24\) \\
\quad Positional encoding & fixed sinusoidal \\
\quad Transformer block & pre-norm self-attention plus feed-forward residual block \\
\quad FFN (multiplier, width) & \(2, \ 192\) \\
\quad Activation & GELU in FFN \\
\quad Dropout & \(0.05\) \\[0.3em]

\multicolumn{2}{l}{\textit{Residual output}} \\
\quad Output form & \(\widehat z_t=z_t^{\rm c}+g_t\odot \delta_t\) \\
\quad Residual head & \(\delta_t=W_\delta h_t+b_\delta\in\mathbb R^3\) \\
\quad Gate head & \(g_t=\operatorname{sigmoid}(W_g h_t+b_g)\in(0,1)^3\) \\
\quad Gate bias initialization & \(b_g=-2.0\) \\[0.3em]

\multicolumn{2}{l}{\textit{FiLM conditioning}} \\
\quad FiLM location & after attention encoder, before residual output heads \\
\quad FiLM map & \(h_t^{\rm FiLM}=(1+\gamma_t)\odot h_t+\beta_t\) \\
\quad Context MLP (width, activation) & \((48, \ \tanh)\) \\
\quad Mean/std/eigs context & \((\mu_t,s_t,\lambda_1(t),\lambda_2(t))\in\mathbb R^8\) \\
\quad Eigenvector context & leading covariance eigenvector(s), sign-aligned in time \\
\quad FiLM output width & \(2d_{\rm hid}\), producing \((\gamma_t,\beta_t)\) \\

\multicolumn{2}{l}{\textit{Optimization}} \\
\quad Optimizer & Adam \\
\quad (Epochs, Batch size, Learning rate) & \((40, \ 64, \ 8\times10^{-4})\) \\
\quad Checkpoint selection & lowest validation objective \\[0.3em]

\multicolumn{2}{l}{\textit{Loss weights}} \\
\quad Total loss & \(L_{\rm path}+\lambda_\delta L_\delta+\lambda_{\rm mom}L_{\rm mom}+\lambda_{\rm tail}L_{\rm tail}\) \\
\quad Residual penalty \(\lambda_\delta\) & \(10^{-3}\) \\
\quad Moment penalty \(\lambda_{\rm mom}\) & \(5\times10^{-2}\) \\
\quad Tail penalty \(\lambda_{\rm tail}\) & \(5\times10^{-2}\) \\
\quad Coordinate weights & \((1.5,1.5,3.0)\) \\
\quad Tail (emphasis \(w_{\rm tail}\), \ threshold \(c_{\rm tail}\) )& (\(3.0\), \, \(2.0\) normalized standard deviations ) \\
\quad Smooth exceedance temperature & \(0.05\) \\

\hline
\end{tabular}
\end{table}

\subsection{Evaluation metrics.}
\label{app:lowdim_metrics}

We assess distributional fidelity using the Kullback--Leibler divergence and the log-density \(L^1\) difference, computed from empirical densities with common binning across the ground truth ground truth, imperfect model, baseline, and FiLM-corrected trajectories:
\begin{equation*}
     D_{\rm KL}
    =
    \sum_{b=1}^{B}
    p^{\rm ref}_{b}
    \log
    \frac{p^{\rm ref}_{b}}{p^{\rm model}_{b}}
\quad \text{and} \quad
    L_{\log}
    =
    \sum_{b=1}^{B}
    \left|
        \log p^{\rm ref}_{b}
        -
        \log p^{\rm model}_{b}
    \right|.
\end{equation*}
For each coordinate \(i\), let \(c_{i,99}\) be the 99th percentile of \(|z_i^{\rm ref}|\) on the test trajectory. We report
\begin{equation*}
    E_{99}(z_i)=
\left|
\widehat{\mathbb P}_{\rm ref}(|z_i|>c_{i,99})
-
\widehat{\mathbb P}_{\rm model}(|z_i|>c_{i,99})
\right|.
\end{equation*}

For the third coordinate, where exponential growth is most pronounced in the low-dimensional system, we additionally report the coordinate-wise log-density discrepancy \(L_{\log}(z_3)\) and the $99$th-percentile exceedance-frequency error $E_{99}(z_3)$ separately in Table~\ref{tab:app:lowdim_z3_sweep_summary}. Table~\ref{tab:lowdim_attention_ablation} reports the full-system \(D_{\rm KL}\), \(L_{\log}\), and \(E_{\rm freq}(99)\), confirming that the correction improves the overall distribution rather than only the coordinate \(z_3\).

\subsection{Low-dimensional results tables.}
\label{app:lowdim_results}

Tables~\ref{tab:app:lowdim_z3_sweep_summary}--\ref{tab:lowdim_attention_ablation} are summarizing results for
when the residual-attention trunk, optimizer, gated residual output, and loss are held fixed, while FiLM injects one of two predeclared ensemble-derived instability contexts.
\begin{table}[H]
\centering
\small
\caption{
Sweep-selected low-dimensional $z_3$ tail diagnostics. Entries are mean $\pm$ standard error over five runs. This table is a sensitivity summary, not the primary model-selection protocol.
}
\label{tab:app:lowdim_z3_sweep_summary}
\begin{tabular}{cllccc}
\toprule
$T_{\rm train}$ & Diagnostic & No-FiLM attention & Best FiLM setting & Best FiLM value & Reduction \\
\midrule
5
& $L_{\log}(z_3)$
& $6.985 \pm 0.135$
& eigvec, $\sigma=5{\times}10^{-6}$, $M=16$
& $6.952 \pm 0.195$
& $0.5\%$ \\
5
& $E_{99}(z_3)$
& $0.0042 \pm 0.0005$
& eigvec, $\sigma=10^{-5}$, $M=4$
& $0.0039 \pm 0.0011$
& $6.8\%$ \\
150
& $L_{\log}(z_3)$
& $6.903 \pm 0.100$
& eigvec, $\sigma=10^{-6}$, $M=2$
& $6.727 \pm 0.067$
& $2.5\%$ \\
150
& $E_{99}(z_3)$
& $0.0056 \pm 0.0017$
& mean/std/eigs, $\sigma=10^{-5}$, $M=2$
& $0.0041 \pm 0.0016$
& $26.3\%$ \\
\bottomrule
\end{tabular}
\end{table}

\begin{table}[H]
\centering
\small
\caption{
Low-dimensional residual-attention ablation at $T_{\rm train}=50$, $\sigma=10^{-5}$, and $M=4$. All entries are mean $\pm$ standard error over $n=5$ independent training runs. $E_{99}$ is the absolute 99th-percentile exceedance-frequency error. Lower is better. Bold denotes the best value among the no-FiLM baseline and the two FiLM-conditioned variants.
}
\label{tab:lowdim_attention_ablation}
\begin{tabular}{llccc}
\toprule
Coordinate & Model & $D_{\rm KL}$ $\downarrow$ & $L_{\log}$ $\downarrow$ & $E_{99}$ $\downarrow$ \\
\midrule
$z_1$ & Residual attention
& $2.336 \pm 0.243$
& $7.103 \pm 0.444$
& $0.0084 \pm 0.0023$ \\
$z_1$ & FiLM attention, mean/std/eigs
& $\mathbf{1.884 \pm 0.025}$
& $\mathbf{6.237 \pm 0.110}$
& $0.0056 \pm 0.0015$ \\
$z_1$ & FiLM attention, eigvec
& $2.048 \pm 0.341$
& $6.745 \pm 0.592$
& $\mathbf{0.0048 \pm 0.0007}$ \\
\midrule
$z_2$ & Residual attention
& $2.728 \pm 0.064$
& $5.755 \pm 0.251$
& $0.0156 \pm 0.0017$ \\
$z_2$ & FiLM attention, mean/std/eigs
& $2.430 \pm 0.238$
& $6.015 \pm 0.455$
& $0.0137 \pm 0.0017$ \\
$z_2$ & FiLM attention, eigvec
& $\mathbf{2.112 \pm 0.107}$
& $\mathbf{4.993 \pm 0.205}$
& $\mathbf{0.0112 \pm 0.0018}$ \\
\midrule
$z_3$ & Residual attention
& $1.226 \pm 0.109$
& $7.231 \pm 0.228$
& $0.0042 \pm 0.0005$ \\
$z_3$ & FiLM attention, mean/std/eigs
& $\mathbf{1.210 \pm 0.159}$
& $\mathbf{6.962 \pm 0.164}$
& $0.0063 \pm 0.0012$ \\
$z_3$ & FiLM attention, eigvec
& $1.451 \pm 0.153$
& $7.330 \pm 0.176$
& $\mathbf{0.0039 \pm 0.0011}$ \\
\bottomrule
\end{tabular}
\end{table}

\FloatBarrier

%% file: appendix_sections/QG_model.tex
\subsection{Full model definition}

We consider the baroclinic two-layer quasi-geostrophic (QG) system on the doubly periodic square domain
\(\Omega = [0,2\pi]^2,\)
with streamfunctions \(\psi_1(x,y,t)\) and \(\psi_2(x,y,t)\) in the upper and lower layers, respectively. The corresponding potential vorticities are
\begin{equation*}
q_1 = \nabla^2 \psi_1 + \frac {k_d^2} 2(\psi_2-\psi_1),
\qquad
q_2 = \nabla^2 \psi_2 + \frac {k_d^2} 2(\psi_1-\psi_2) + \frac{f_0}{h_2} h_b(x,y),
\end{equation*}
where \(\Delta=\partial_{xx}+\partial_{yy}\), \(k_d^2\) is the deformation wavenumber parameter, \(\beta\) is the beta-plane parameter, and \(h_b\) is a prescribed bottom-topography field acting only in the lower layer, with $f_0= 1$ the inertial frequency and $h_2$ the bottom-layer thickness. The layer velocities are
\begin{equation*}
\mathbf{u}_i = U_i\textbf{e}_x + \hat{\textbf{k}} \times \nabla \psi_i, \qquad i\in\{1,2\},
\end{equation*}
where $\hat{\textbf{k}}$ is the unit vector orthogonal to the $(x,y)$ plane.
The governing equations are
\begin{equation}
\partial_t q_i + \mathbf{u}_i\cdot \nabla q_i + (\beta+ k_d^2 U_i)\,\partial_x \psi_i = -\delta_{2, i}\ r \nabla ^2\psi_i -  \nu \Delta^4q_i, \qquad i=1,2.
\label{eq:qg}
\end{equation}
For this implementation, the model is driven to a statistically stationary chaotic regime by a constant zonal mean shear and by bottom topography in the lower layer. The streamfunctions are recovered from the PVs by inversion of the elliptic relations above.

This model approximates mid- and high-latitude atmospheric circulation subject to an imposed shear current. The chosen parameterization is consistent with the radius and rotation of the earth and the characteristic length and velocity scales of the atmosphere, for a complete list of values, refer to Table~\ref{app:QG:parameter_table}.

We further introduce a bottom topography that is prescribed as a periodized sum of seven
Gaussian bumps,
\[
h_b(x,y)
=
h_0\sum_{j=1}^{7}
\exp\!\left(
-\frac{(x-a_j)^2+(y-b_j)^2}{\sigma_h^2}
\right),
\]
with \((a_j,b_j)\) chosen to respect periodicity (see Fig ~\ref{fig:app_QGtopograpy} for visual depiction of the introduced topography). T
his topography breaks the spatial homogeneity of the flow, induces strongly anisotropic large-scale structures breaks spatial homogeneity, and generates localized extreme events that less pronounced in the corresponding flat-bottom formulation. As a result, the benchmark is substantially more demanding than standard two-layer QG test cases without bottom topography, and it provides a stricter assessment of whether a surrogate can recover rare-event statistics under spatially inhomogeneous dynamics.

\begin{figure}
    \centering
    \includegraphics[width=0.75\linewidth]{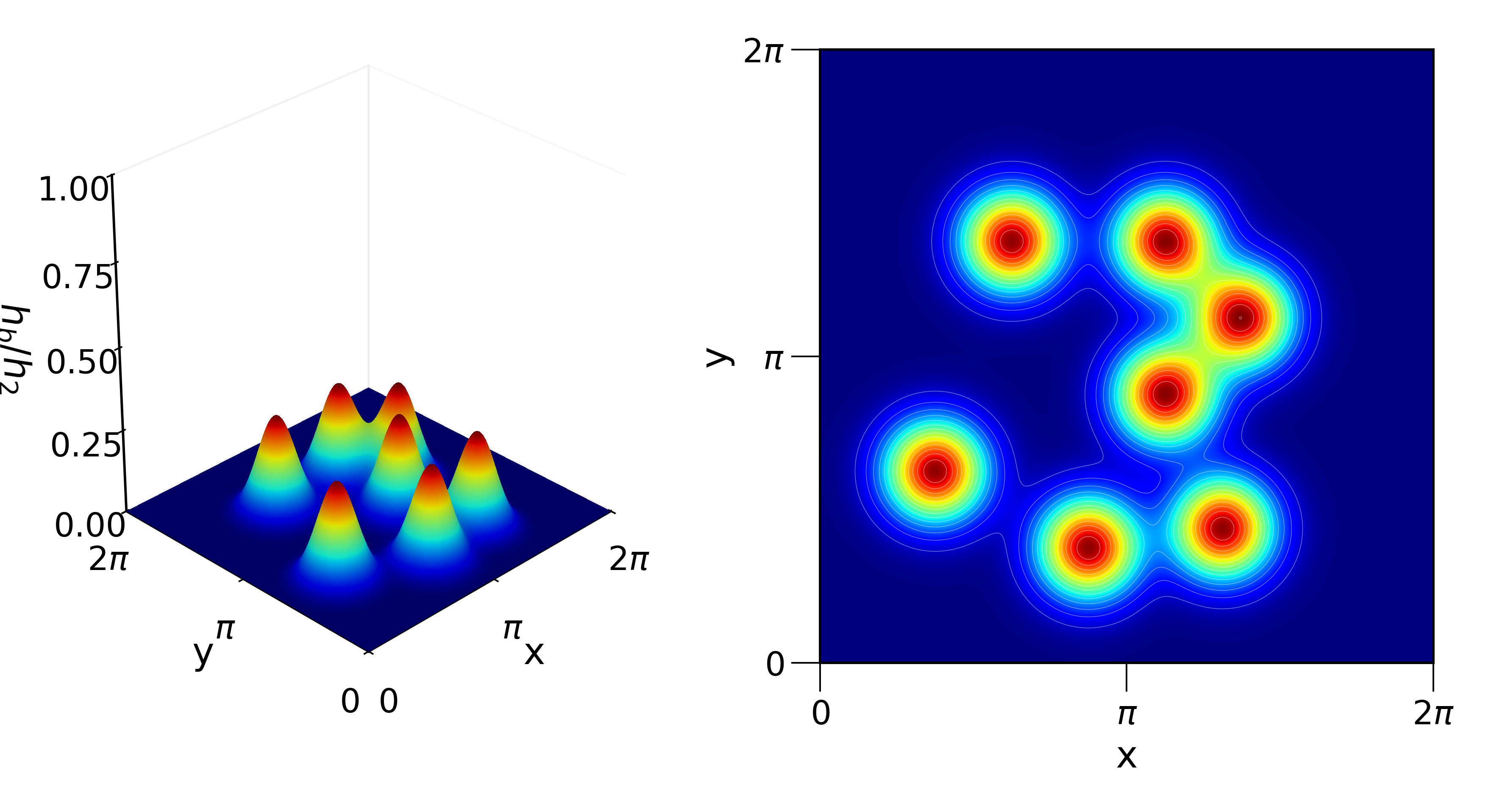}
    \caption{Bottom topography used in the quasi-geostrophic benchmark. The lower-layer forcing is constructed from seven Gaussian bumps, producing spatially inhomogeneous and anisotropic dynamics that are substantially more challenging than standard two-layer QG test cases without topography.}
    \label{fig:app_QGtopograpy}
\end{figure}

\subsection{Ensemble formulation}

Let \(q^{\mathrm{ref}}(t)\) denote the projected high-resolution reference or ground truth PV on the model grid and \(q^\tau(t)\) the nudged trajectory of the imperfect coarser model. This nudged model is defined by
\begin{equation}
\partial_t q^\tau_i + \mathbf{u}^\tau_i\cdot \nabla q^\tau_i + (\beta\,+ k^2_dU_i)\partial_x \psi^\tau_i
=-\delta_{2, i}\ r \nabla ^2\psi^\tau_i -  \nu \Delta^4q^\tau_i
-\tau^{-1}\!\left(q^\tau_i - \mathcal{P}\textbf{q}^{\mathrm{ref}}_i\right),\qquad i=1,2,
\label{eq:nudged_qg}
\end{equation}
where \(\mathcal{P}\) denotes the spectral projector onto the controlled coarse modes. In Fourier variables, the nudging acts algebraically on each resolved mode \(k\) as
\vspace{-8pt}
\begin{equation*}
\widehat{(q^\tau_i-\mathcal{P}\textbf{q}^{\mathrm{ref}}_i)}(k)
=
\mathbf{1}_{\{|k_x|\le k_{\max},\,|k_y|\le k_{\max}\}}
\left(\hat q^\tau_i(k)-\hat q^{\mathrm{ref}}_i(k)\right).
\end{equation*}
The ensemble members are independent realizations of a stochastic perturbation of the nudged imperfect model:
\vspace{-8pt}
\begin{equation}
\begin{split}
\mathrm d q_i^{(m)}
={}&
\Big[
-\mathbf u_i^{(m)}\cdot\nabla q_i^{(m)}
-\bigl(\beta+k_d^2U_i\bigr)\partial_x\psi_i^{(m)}
-\delta_{i2}r\Delta\psi_i^{(m)}
-\nu\Delta^4q_i^{(m)}  \\
&\hspace{2.5cm}
-\tau^{-1}\!\left(q_i^{(m)}-\mathcal{P}(\textbf{q}_i^{\mathrm{ref}}\right)
\Big]\mathrm dt
+\sigma\,\mathrm dW_i^{(m)},
\qquad i=1,2,
\end{split}\label{eq:ensemble_qg}
\vspace{-28pt}
\end{equation}
for \(m=1,\dots,M\), where \(W^{(m)}_i\) are independent Wiener processes applied mode-wise on the resolved Fourier coefficients. The ensemble is uncoupled conditional on the nudged base trajectory: members interact only through the shared reference \(q^{\mathrm{ref}}\) and the same nudging operator, not through direct member-to-member coupling. The stochastic term is therefore additive and independent across members and layers. The ensemble size used in the QG experiments is \(M\in\{1,2,4,8\}\). The results show that \(M\le 4\) is sufficient in the tested regime, with correspondingly low computational overhead.

\subsection{Numerical discretization and QG solver pipeline}
\vspace{-0.1cm}
The QG equations are discretized on a uniform doubly periodic grid with either \(128\times128\) or \(24\times24\) points, corresponding respectively to the ground truth/reference and coarse model simulations. Spatial derivatives are evaluated spectrally by Fourier collocation on the periodic domain. 
For \(k\neq0\), let \(\kappa=|k|^2\). The Fourier-space PV--streamfunction
relations are

\begin{equation*}
    \begin{pmatrix}
\hat q_1(k)\\
\hat q_2(k)-\widehat{\frac{f_0}{h_2}h_b(x,y)}(k)
\end{pmatrix}
=
-
\begin{pmatrix}
\kappa+\frac{k_d^2}2& -\frac{k_d^2}2\\
-\frac{k_d^2}2 & \kappa+\frac{k_d^2}2
\end{pmatrix}
\begin{pmatrix}
\hat\psi_1(k)\\
\hat\psi_2(k)
\end{pmatrix}.
\end{equation*}
\vspace{-0.1cm}
Therefore
\vspace{-0.0cm}
\begin{equation*}
    \begin{pmatrix}
\hat\psi_1(k)\\
\hat\psi_2(k)
\end{pmatrix}
=
-\frac{1}{\kappa(\kappa+k_d^2)}
\begin{pmatrix}
\kappa+\frac{k_d^2}2 & \frac{k_d^2}2\\
\frac{k_d^2}2 & \kappa+\frac{k_d^2}2
\end{pmatrix}
\begin{pmatrix}
\hat q_1(k)\\
\hat q_2(k)-\widehat{\frac{f_0}{h_2}h_b(x,y)}(k)
\end{pmatrix},
\qquad k\neq0.
\end{equation*}

with the zero mode treated consistently with the imposed mean flow. A fourth-order Runge--Kutta scheme is used for time integration with time step \(\Delta t = 10^{-2}\), and that the stochastic perturbation is discretized with per-mode variance proportional to \(\sigma^2 \Delta t\).
 
The QG data pipeline is as follows. First, a high-resolution ground truth or reference trajectory \(q^{\mathrm{ref}}_{\mathrm{HR}}(t)\) is generated by integrating the two-layer QG PDE on the \(128\times128\) grid. Second, this trajectory is spectrally projected to the coarse model-\(24\times24\) resolved subspace to obtain \(q^{\mathrm{ref}}(t)\). Third, a imperfect free-running trajectory \(q^{\mathrm{model}}(t)\) is simulated on the coarse grid with inherited initial condition of the projected reference. Fourth, a nudged trajectory \(q^\tau(t)\) following \eqref{eq:nudged_qg} and in parallel an ensemble of nudged realizations \(q^{(m)}(t)\) are evolved from \eqref{eq:ensemble_qg} on the coarse-model grid directly. We note that simualtion on the coarse grid is much cheaper than highly resolved ground truth or reference simulation. For all trajectories initial spin-up period is discarded to remove transient effects from random initialization. 

\subsection{Parameter table}
\vspace{-0.4cm}
\begin{table}[H]
\centering
\small
\caption{Configuration of the two-layer quasi-geostrophic (QG) model with Gaussian bump topography. Parameters are chosen to approximate mid- and high-latitude atmospheric circulation subject to an imposed shear current.}
\label{app:QG:parameter_table}
\begin{tabular}{ll}
\hline
\textbf{Parameter} & \textbf{Value / Description} \\
\hline

\multicolumn{2}{l}{\textit{Model setup}} \\
\quad Spatial domain & $[0, 2\pi]^2$ (doubly periodic) \\
\quad Layers & 2 \\
\quad State variables & $\psi_1, \psi_2, q_1, q_2$ \\[0.3em]

\multicolumn{2}{l}{\textit{Physical parameters}} \\
\quad Bottom drag $r$ & $0.1$ \\
\quad Beta parameter $\beta$ & $2.0$ \\
\quad Deformation parameter $k_d^2$ & $4.0$ \\
\quad Inertial frequency $f_0$ & $1.0$\\
\quad Mean flow $(U_1, U_2)$ & $(0.2, -0.2)$ \\
\quad Bottom topography $h_b$ & 7 Gaussian bumps (fixed locations, $\sigma_h^2=0.25$, $h_0=0.3$) \\[0.3em]

\multicolumn{2}{l}{\textit{Discretization and numerics}} \\
\quad Ground truth resolution & $128 \times 128$ \\
\quad Projected resolution & $24 \times 24$ \\
\quad Time step $\Delta t$ & $10^{-2}$ \\[0.3em]

\multicolumn{2}{l}{\textit{Nudging / ensemble}} \\
\quad Nudging timescale $\tau$ & $16.0$ \\
\quad Ensemble size $M$ & $1, 2, 4, 8$ \\
\quad Noise amplitude $\sigma$ & $10^{-3}, 10^{-2}, 10^{-1}, 0.2, 0.4$ \\

\hline
\end{tabular}
\end{table}

\FloatBarrier

%% file: appendix_sections/ML.tex
\subsection{Problem formulation and model summary}
\label{app:ml:problem}

From the machine-learning perspective, the task is supervised sequence-to-sequence regression with probabilistic modeling. The goal is to learn a correction operator that maps nudged trajectories of the imperfect model of the QG system to their projected high-resolution counterparts over finite temporal windows and then generalizes to much longer free-running and biased model trajectories at test time.

Let \(v_{\tau, t} \in \mathbb{R}^{d}\) denote the nudged model state, \(u_t \in \mathbb{R}^{d}\) the projected high-resolution ground truth, with \(d=2N_xN_y\) the flattened two-layer field dimension. Training examples are overlapping windows
\[
X_k=
\bigl(v_\tau(t_k+j\Delta t_{\mathrm{ML}})\bigr)_{j=0}^{H-1},
\qquad
Y_k=
\bigl(u(t_k+j\Delta t_{\mathrm{ML}})\bigr)_{j=0}^{H-1}.
\]
Here, \(H\) is the number of sampled time steps per window and \(t_k\) is its start time and stride \(s\) determined by the ML time step \(\Delta t_{\mathrm{ML}}\). For ensemble-informed models, each time step carries a context vector \(c_t \in \mathbb{R}^{D_c}\) constructed from the ensemble, and the corresponding windowed context is
\[
C^{(\tau)} = (c_{\tau}, c_{\tau+s}, \dots, c_{\tau+(H-1)s}) \in \mathbb{R}^{H \times D_z}.
\]

The correction operator is then a parametric map
\[
G_\theta : (X_k^{(\tau)}, C_k^{(\tau)}) \mapsto \hat Y_k^{(\tau)},
\]
with \(C^{(\tau)}\) omitted for the baseline. Training minimizes a single objective \(\mathcal{L}(\theta)\) defined in Section~\ref{app:ml:loss}, and evaluation uses long-horizon statistics of the corrected imperfect model trajectory \(\hat u_t = G_\theta(v_t, \cdot)\) over a test horizon \(T_{\mathrm{test}} \gg T_{\mathrm{train}}\).


\FloatBarrier
\subsection{Model architecture}
\label{app:ml:model}

\paragraph{Baseline: STORN backbone}
The baseline model is a stochastic recurrent neural network of STORN type. We use the implementation of \cite{barthelsorensen2024} from its Zenodo release (version v1, \cite{barthelsorensen2024a}). For each time step \(t\) in a window \(X=(x_1,\dots,x_H)\),

\begin{enumerate}
\item An encoder maps \(x_t\) to a hidden representation \(h^{(0)}_t\).
\item A Gaussian latent posterior with mean \(\mu_t\) and scale \(\sigma_t\) is produced from \(h^{(0)}_t\), and a latent sample is drawn via the reparameterization trick,
\[
z^{\text{latent}}_t = \mu_t + \sigma_t \odot \epsilon_t,\quad \epsilon_t \sim \mathcal{N}(0,I).
\]
\item The concatenated state \(\phi_t = [h^{(0)}_t, z^{\text{latent}}_t]\) is passed through a recurrent trunk (LSTM) across the window.
\item A linear decoder maps the recurrent hidden state to the corrected field \(\hat y_t\).
\end{enumerate}

The model outputs both \(\hat y_t\) and the latent parameters \((\mu_t, \sigma_t)\), which are used in the KL term of the training objective. Concrete dimensionalities (encoder width, latent size, recurrent width) are summarized in Table~\ref{tab:ml_hyperparams}.

The following pseudocode summarizes the training and evaluation pipelines. It omits mathematical detail already defined above and focuses on inputs, key transformations, and outputs.

\begin{algorithm}[t]
\caption{STORN training}
\label{alg:storn_train}
\begin{algorithmic}[1]
\Require nudged trajectory \(v_\tau(t)\), ground truth \(u(t)\), window length \(H\), stride \(s\), batch size \(B\)
\Ensure trained STORN model \(G_\theta\)
\State split windows into train/validation (90\%/10\% in time)
\State extract windows \((X_k,Y_k)\) from \((v_\tau,u)\) using \(H,s\)
\State initialize STORN parameters \(\theta\)
\For{epochs}
    \For{batches of windows}
        \State compute \(\hat Y = G_\theta(X)\) and latent parameters \((\mu,\sigma)\)
        \State evaluate \(\mathcal{L}(\theta)\) from Eq.~\eqref{eq:ml_loss}
        \State update \(\theta\) with Adam
    \EndFor
\EndFor
\end{algorithmic}
\end{algorithm}

\paragraph{Ensemble-conditioned model: FiLM-STORN.}
The FiLM-STORN extends the backbone by conditioning the recurrent trunk on ensemble-derived context. For each time step \(t\),
\begin{itemize}
\item A context vector \(c_t \in \mathbb{R}^{D_z}\) (see Section~\ref{app:ml:context}) is computed from the ensemble.
\item A small conditioning network maps \(c_t\) to FiLM parameters \((\gamma_t, \beta_t)\), with the same width as the recurrent hidden state.
\item The recurrent hidden state \(h_t\) is modulated as
\[
\tilde h_t = (1 + \gamma_t) \odot h_t + \beta_t,
\]
before being fed to the decoder.
\end{itemize}
This preserves the probabilistic structure of STORN while coupling the dynamics to low-dimensional ensemble context. The context dimension and FiLM widths are specified in Table~\ref{tab:ml_hyperparams}.

\paragraph{Context construction.}\label{app:ml:context}
For each layer \(\ell\in\{1,2\}\) and ensemble member \(m\),
define the spatially averaged amplitude
\[
a_\ell^{(m)}(t)
=
\frac1{|\Omega|}
\int_\Omega |\psi_\ell^{(m)}(x,t)|\,\mathrm dx.
\]
The training context is
\[
\mu_\ell(t)=\frac1M\sum_{m=1}^M a_\ell^{(m)}(t),
\qquad
s_\ell(t)=
\left[
\frac1{M-1}\sum_{m=1}^M
\bigl(a_\ell^{(m)}(t)-\mu_\ell(t)\bigr)^2
\right]^{1/2},
\]
\[
c_t=(\mu_1(t),s_1(t),\mu_2(t),s_2(t))
\in\mathbb R^4.
\]
These summaries describe layer amplitudes and ensemble spread;
they do not retain the full spatial covariance eigenspace.

At deployment, let \(a_\ell^{\rm c}(t_n)\) denote the same
spatial amplitude computed from the free-running and imperfect model field.
Using \(W=32\) sampled time steps, define
\[
\widetilde\mu_{\ell,n}
=\frac1W\sum_{j=0}^{W-1}a_\ell^{\rm c}(t_{n-j}),
\]
\[
\widetilde \sigma_{\ell,n}
=
\left[
\frac1{W-1}\sum_{j=0}^{W-1}
\bigl(a_\ell^{\rm c}(t_{n-j})
-\widetilde\mu_{\ell,n}\bigr)^2
\right]^{1/2},
\]
\[
\widetilde c_n=
(\widetilde\mu_{1,n},\widetilde \sigma_{1,n},
 \widetilde\mu_{2,n},\widetilde \sigma_{2,n}).
\]
Training statistics measure variation across ensemble members, whereas deployment statistics measure variation over time along one imperfect model trajectory. At deployment, the QG context is computed from the free-running model trajectory only; this rolling statistic is not an instantaneous ensemble covariance, but a leakage-free proxy for the low-order instability-intensity regime learned from the nudged ensemble during training.

\begin{algorithm}[t]
\caption{FiLM-STORN training}
\label{alg:film_train}
\begin{algorithmic}[1]
\Require nudged trajectory \(v_\tau(t)\), ground truth \(u(t)\), ensemble members \(\{Q^{(m)}(t)\}\), window length \(H\), stride \(s\), batch size \(B\)
\Ensure trained FiLM-STORN model \(G_\theta\)
\State compute context sequence \(c_t\) from \(\{Q^{(m)}(t)\}\)
\State extract windows \((X_k,Y_k,C_k)\) from \((v_\tau,u,c)\) using \(H,s\)
\State split windows into train/validation (90\%/10\% in time)
\State initialize FiLM-STORN parameters \(\theta\)
\For{epochs}
    \For{batches of windows}
        \State compute \(\hat Y = G_\theta(X,C)\) and latent parameters \((\mu,\sigma)\)
        \State evaluate \(\mathcal{L}(\theta)\) from Eq.~\eqref{eq:ml_loss}
        \State update \(\theta\) with Adam
    \EndFor
\EndFor
\end{algorithmic}
\end{algorithm}



\subsection{Training objective}
\label{app:ml:loss}

All models are trained to minimize a single semi-physics-informed objective combining reconstruction accuracy, mass conservation, and KL regularization. For a batch of \(N\) windows \(\{(X_k,Y_k,C_k)\}_{k=1}^N\), the loss is
\begin{equation}
\mathcal{L}(\theta)
=
\frac{1}{N}\sum_{k=1}^{N}
\sum_{t=1}^{H}
\Big[
\| \hat{y}_{k,t} - y_{k,t} \|_2^2
+
\lambda_{\mathrm{mass}} \, \mathcal{L}_{\mathrm{mass}}(\hat{y}_{k,t})
\Big]
+
\lambda_{\mathrm{KL}} \, \mathcal{L}_{\mathrm{KL}}.
\label{eq:ml_loss}
\end{equation}
with mass penalty 
\begin{equation*}
\mathcal{L}_{\mathrm{mass}}(\hat{y}_t)
=
\sum_{\ell=1}^2
\left|
\sum_{i=1}^{N_x N_y} \hat{y}_t^{(\ell,i)}
\right|,
\end{equation*}
and the KL term uses the closed-form divergence between the Gaussian latent posterior and a standard normal prior,
\begin{equation*}
\mathcal{L}_{\mathrm{KL}}=\frac{1}{N}\sum_{k,t}\frac{1}{2}\sum_{j}\left(
\sigma_{k,t,j}^2 + \mu_{k,t,j}^2- 1 - \log (\sigma_{k,t,j}^2 + \varepsilon)
\right),
\end{equation*}
with a small \(\varepsilon\) for numerical stability. The mass conservation term arises from the linear dependence of the streamfunctions on the layer height perturbations, combined with the volume conservation constraint, which enforces that the domain-integrated height disturbance is zero

\subsection{Implementation details: data, hyperparameters and optimization}
\label{app:ml:impl}

\paragraph{Data and windowing.}
After discarding an initial spin-up period, a single long reference realization is used to construct training windows of length \(H\) at temporal stride \(s\). For each training horizon \(T_{\mathrm{train}}\), windows are extracted from the corresponding truncated segment. If the resulting number of windows falls below a small threshold (we use 10), \(H\) is reduced until enough windows are available. The train/validation split is 90\%/10\% in temporal order; windows are not shuffled to preserve sequential structure.

\paragraph{Batching and optimization.}
We deliberately employ a minimal and stable optimization setup in order to isolate the effect of the proposed architectural components. Adam is used for all models, with learning rates and number of epochs specified in Table~\ref{tab:ml_hyperparams}. We found this setup to be sufficient for stable convergence across all configurations, and it avoids introducing additional hyperparameters that could confound comparisons between models.

\paragraph{Hardware and runtime.}
All experiments are run on CPU-based workers in a local cluster. Runtime and peak memory increase monotonically with the number of additional ensemble members \(M\), but the FiLM-STORN input dimension remains independent of \(M\) because only low-order statistics enter the context. A full training run for a single configuration completes within a few hours on a modern multi-core CPU; see Figure ~\ref{fig:app:runtime}

\begin{figure}
    \centering
    \includegraphics[width=1\linewidth]{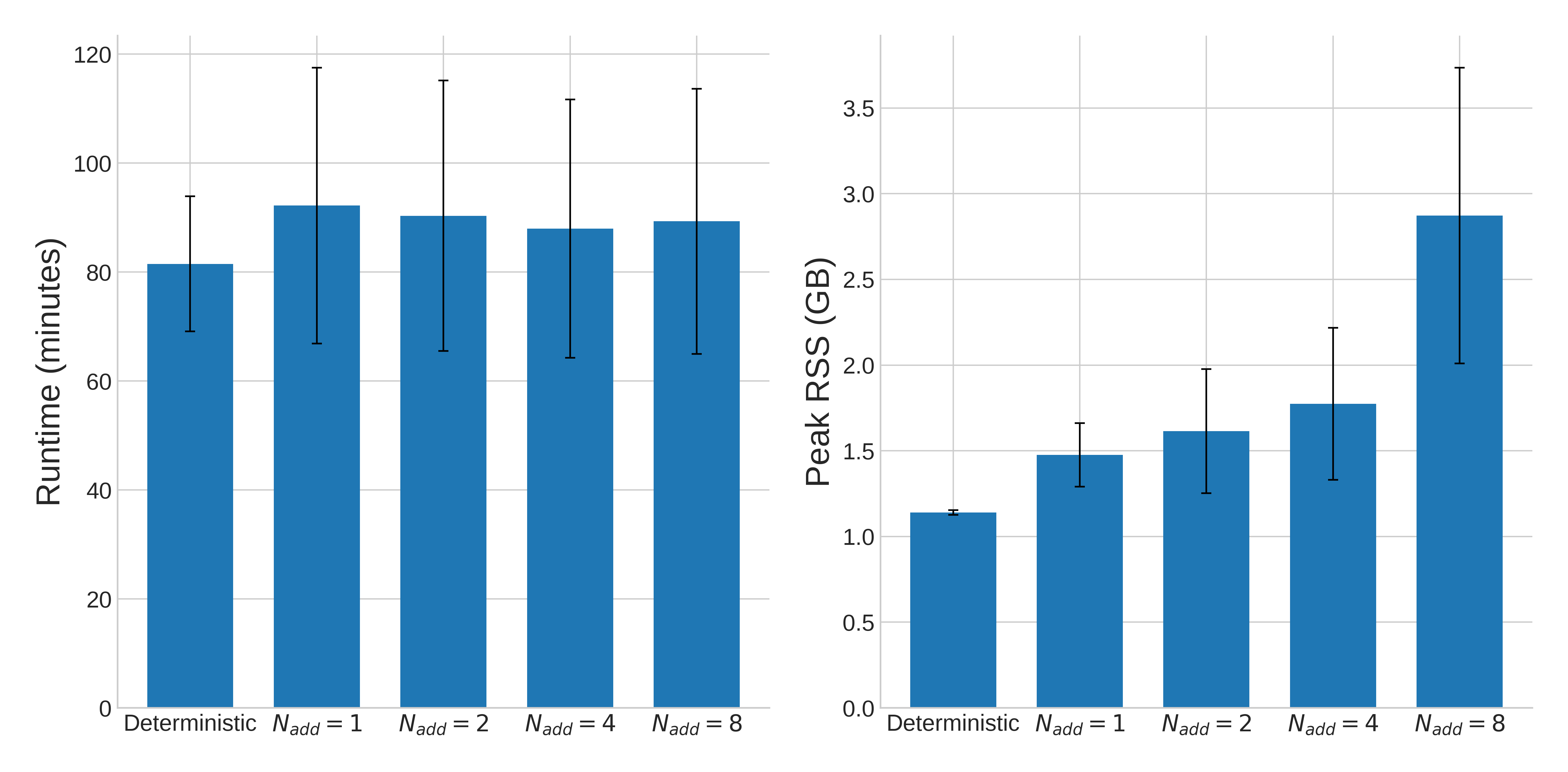}
    \caption{Runtime and peak memory versus number of additional ensemble members $M$. Vertical bars indicate variability across noise levels $\sigma$.}
    \label{fig:app:runtime}
\end{figure}

\paragraph{Hyperparameters for QG implemetation}
\begin{table}[!ht]
\centering
\small
\caption{Machine learning hyperparameters and architecture settings for the QG experiments.}
\begin{tabular}{l p{10cm}}
\hline
\textbf{Component} & \textbf{Value / Description} \\
\hline

\multicolumn{2}{l}{\textit{Data and sequence setup}} \\
\quad Grid resolution & \(N_x=N_y=24\), two layers, \(d=2N_xN_y=1152\) \\
\quad Window length \(H\) & 100, (adaptively reduced if too few windows are available, i.e. \(T_{\mathrm{train}} = 50\) \\
\quad ML time step \(\Delta t_{\mathrm{ML}}\) & 0.1, corresponding to stride \(s=\Delta t_{\mathrm{ML}}/\Delta t_{\mathrm{DNS}}=10\) \\
\quad Training horizons & \(T_{\mathrm{train}} \in \{50,200,1000\}\) \\
\quad Test horizon & \(T_{\mathrm{test}}=50\,000\) \\
\quad Train/val split & 90\% / 10\%, temporal split \\
\hline

\multicolumn{2}{l}{\textit{Architecture: baseline STORN}} \\
\quad Input representation & Flattened two-layer streamfunction fields of nudged and ground truth fields for training \\
\quad Encoder width & 60 \\
\quad Latent dimension \(L\) & 60 \\
\quad Recurrent trunk width & 60 \\
\quad Activation functions & \(\tanh\) in dense/LSTM layers, hard-sigmoid recurrent gate \\
\hline

\multicolumn{2}{l}{\textit{Architecture: FiLM-STORN}} \\
\quad Conditioning input & Training: layerwise amplitude statistics from the nudged auxiliary ensemble. Test: causal layerwise amplitude statistics from the free-running model trajectory. No ground truth information is used at test time. \\
\quad Context dimension \(D_z\) & 4: \((\mu_1,\sigma_1,\mu_2,\sigma_2)\) \\
\quad FiLM embedding width & 32 \\
\quad FiLM modulation & \(h \mapsto (1+\gamma(z))\odot h + \beta(z)\) \\
\quad FiLM output width & 120 ($2h_{hidden}$) \\
\hline

\multicolumn{2}{l}{\textit{Training}} \\
\quad Optimizer & Adam \\
\quad Batch size & 32 \\
\quad \(\lambda_{\mathrm{mass}}, \ \lambda_{\mathrm{KL}}\) & \(10^{-4}\) \\
\quad Mass penalty \(\lambda_{\mathrm{mass}}\) & fixed per QG configuration \\
\hline

\multicolumn{2}{l}{\textit{Ensemble conditioning}} \\
\quad Additional ensemble sizes \(M\) & \(0,1,2,4,8\) \\
\quad Ensemble context construction & Per-time-step mean and standard deviation of layerwise amplitudes \\
\hline
\end{tabular}
\label{tab:ml_hyperparams}
\end{table}

\FloatBarrier
\subsection{Evaluation protocol}
\label{app:ml:evaluation}

The trained operator is evaluated on an independent imperfect model trajectory of length \(T_{\mathrm{test}} \gg T_{\mathrm{train}}\). The test sequence \(X_{\mathrm{test}}\) is constructed with the same conventions as in training, but using only free-running coarse model data. This introduces a deliberate proxy shift from training ensemble spread to rolling coarse model spread, chosen to avoid ground truth leakage. While this constitutes a distribution shift, both the nudged-model and the imperfect free-running data evolve on the same statistical attractor. As a result, we do not interpret the rolling statistic as a transverse covariance estimator, but as a low-order proxy for the instability-intensity regime. Its validity relies on the coarse model dynamics retaining a consistent signature of temporally localized sensitivity, so that coarse model fluctuations remain informative of the instability structure learned during training.

The corrected trajectory \(\hat u_t = G_\theta(v_t, \cdot)\) is analyzed using long-horizon diagnostics: probability density functions (PDFs) and log-PDF errors exceedance-area statistics, and other derived metrics defined in Section~\ref{app:evaluation_metrics}. All evaluation is offline and the same analysis procedures are applied to the ground truth, imperfect model, STORN, and FiLM-STORN corrected trajectories. 

\begin{algorithm}[t]
\caption{Long-horizon evaluation}
\label{alg:evaluation}
\begin{algorithmic}[1]
\Require trained model \(G_\theta\), imperfect model trajectory \(v(t)\), test horizon \(T_{\mathrm{test}}\)
\Ensure corrected trajectory \(\hat u(t)\), long-horizon diagnostics
\State build \(X_{\mathrm{test}}\) from \(v(t)\) using the same stride \(s\)
 \If{using FiLM-STORN}
\State build context \(C_{\mathrm{test}}\) from layer-wise amplitude and statistics of $v(t)$ only (see Section~\ref{app:ml:context})
\EndIf
\State compute \(\hat u(t) = G_\theta(X_{\mathrm{test}}, C_{\mathrm{test}})\)
\State compute PDFs, log-PDF error, and exceedance statistics for \(\hat u(t)\)
\end{algorithmic}
\end{algorithm}

\subsection{Mechanism controls and ablation study}\label{app:ml:ablation}

We evaluate the controls at the main-paper configuration, \(T_{\mathrm{train}}=50\), \(\sigma=0.01\), and \(M=2\), without retuning the loss, optimizer, training schedule, or test protocol. The STORN backbone is retained, with input dimensions adjusted where required for concatenation. The covariance-direction control uses a leading eigenvector of the centered full-field empirical covariance. The random-direction control uses a dimension- and norm-matched random sequence. The imperfect model-statistics control uses analogous statistics of the coarser model state. Mean-only and variance-only controls test the contributions of ensemble amplitude and spread separately. Finally, the concatenation controls supply the same covariance-direction or mean/spread information without FiLM modulation.

\begin{table}[t]
\centering
\caption{
QG main comparison and mechanism controls.
STORN$^*$ denotes the state-of-the-art literature benchmark
\citep{barthelsorensen2024}, evaluated here with
$T_{\mathrm{train}}=1000$.
All other models use $T_{\mathrm{train}}=50$;
ensemble-based configurations use $\sigma=0.01$ and $M=2$.
Main-comparison values are averages over five independent runs;
mechanism-control values are averages over three independent runs.
Lower is better for all metrics and negative percentages indicate lower error. Exceedance errors are averaged over
$c\in\{1.0,1.5,1.75\}$. Bold values indicate the lowest error within each block.
}
\label{tab:qg_mechanism_controls}

\small
\setlength{\tabcolsep}{3pt}
\renewcommand{\arraystretch}{1.15}
\def\qgcell#1#2{\shortstack{#1\\{\scriptsize #2}}}

\begin{tabular}{@{}lcccccc@{}}
\toprule
Method
& $D_{\mathrm{KL}}\downarrow$
& $L_{\log}\downarrow$
& \shortstack{$\overline{E}_{AOT, \psi_2}$ $\downarrow$}
& \shortstack{$\overline{E}_{freq, \psi_2}$ $\downarrow$}
& \shortstack{$\overline{E}_{AOT}$ $\downarrow$}
& \shortstack{$\overline{E}_{freq}$ $\downarrow$} \\
\midrule

\multicolumn{7}{@{}l}{\textbf{Main comparison}}\\
\addlinespace[3pt]

\shortstack[l]{STORN$^*$\\$T_{\mathrm{train}}=1000$}
& \qgcell{\textbf{0.00393}}{($-89.2\%$)}
& \qgcell{\textbf{5.3454}}{($-35.6\%$)}
& \qgcell{0.00501}{($+28.1\%$)}
& \qgcell{0.10564}{($+53.7\%$)}
& \qgcell{0.01195}{($+14.1\%$)}
& \qgcell{0.11726}{($+62.9\%$)} \\
\addlinespace[4pt]

\shortstack[l]{STORN\\$T_{\mathrm{train}}=50$}
& \qgcell{0.07079}{($+94.2\%$)}
& \qgcell{11.9700}{($+44.3\%$)}
& \qgcell{0.00589}{($+50.6\%$)}
& \qgcell{0.14660}{($+113.3\%$)}
& \qgcell{0.01574}{($+50.3\%$)}
& \qgcell{0.12118}{($+68.3\%$)} \\
\addlinespace[4pt]

\shortstack[l]{FiLM-STORN (ours)\\$T_{\mathrm{train}}=50$}
& \qgcell{0.03646}{(ground truth)}
& \qgcell{8.2959}{(ground truth)}
& \qgcell{\textbf{0.00391}}{(ground truth)}
& \qgcell{\textbf{0.06874}}{(ground truth)}
& \qgcell{\textbf{0.01047}}{(ground truth)}
& \qgcell{\textbf{0.07199}}{(ground truth)} \\

\midrule
\multicolumn{7}{@{}l}{\textbf{Mechanism controls}}\\
\addlinespace[3pt]

\shortstack[l]{Covariance direction\\+ FiLM}
& \qgcell{\textbf{0.02664}}{($-26.9\%$)}
& \qgcell{\textbf{10.8471}}{($+30.8\%$)}
& \qgcell{\textbf{0.02671}}{($+583.1\%$)}
& \qgcell{\textbf{0.03340}}{($-51.4\%$)}
& \qgcell{0.03417}{($+226.4\%$)}
& \qgcell{\textbf{0.07387}}{($+2.6\%$)} \\
\addlinespace[4pt]

\shortstack[l]{Random direction\\+ FiLM}
& \qgcell{0.11841}{($+224.8\%$)}
& \qgcell{17.1789}{($+107.1\%$)}
& \qgcell{0.04412}{($+1028.4\%$)}
& \qgcell{0.13868}{($+101.7\%$)}
& \qgcell{0.05928}{($+466.2\%$)}
& \qgcell{0.18024}{($+150.4\%$)} \\
\addlinespace[4pt]

\shortstack[l]{Coarser model statistics\\+ FiLM}
& \qgcell{0.04113}{($+12.8\%$)}
& \qgcell{18.6588}{($+124.9\%$)}
& \qgcell{0.03696}{($+845.3\%$)}
& \qgcell{0.09232}{($+34.3\%$)}
& \qgcell{0.04492}{($+329.0\%$)}
& \qgcell{0.08514}{($+18.3\%$)} \\
\addlinespace[4pt]

\shortstack[l]{Ensemble mean only\\+ FiLM}
& \qgcell{0.03533}{($-3.1\%$)}
& \qgcell{14.1135}{($+70.1\%$)}
& \qgcell{0.02908}{($+643.7\%$)}
& \qgcell{0.09408}{($+36.9\%$)}
& \qgcell{\textbf{0.03266}}{($+211.9\%$)}
& \qgcell{0.07986}{($+10.9\%$)} \\
\addlinespace[4pt]

\shortstack[l]{Ensemble variance only\\+ FiLM}
& \qgcell{0.18193}{($+399.0\%$)}
& \qgcell{68.2990}{($+723.3\%$)}
& \qgcell{0.09227}{($+2259.8\%$)}
& \qgcell{0.51458}{($+648.6\%$)}
& \qgcell{0.13459}{($+1185.5\%$)}
& \qgcell{0.38857}{($+439.8\%$)} \\
\addlinespace[4pt]

\shortstack[l]{Covariance direction\\+ concatenation}
& \qgcell{0.09419}{($+158.3\%$)}
& \qgcell{19.6264}{($+136.6\%$)}
& \qgcell{0.04621}{($+1081.8\%$)}
& \qgcell{0.12970}{($+88.7\%$)}
& \qgcell{0.05586}{($+433.5\%$)}
& \qgcell{0.14467}{($+101.0\%$)} \\
\addlinespace[4pt]

\shortstack[l]{Mean/spread\\+ concatenation}
& \qgcell{0.07744}{($+112.4\%$)}
& \qgcell{17.5130}{($+111.1\%$)}
& \qgcell{0.04530}{($+1058.6\%$)}
& \qgcell{0.13465}{($+95.9\%$)}
& \qgcell{0.05070}{($+384.2\%$)}
& \qgcell{0.12238}{($+70.0\%$)} \\

\bottomrule
\end{tabular}
\end{table}

The combined mean/spread FiLM context gives lower log-\(L_1\), exceedance-area, and two-layer average frequency errors than any of the mechanism controls. The covariance-direction context gives lower KL and lower-layer frequency errors than the combined context, but larger errors in the remaining diagnostics. Thus, the full covariance direction does not improve every metric. For both information sources tested, FiLM yields lower errors than concatenation across all six diagnostics.

\FloatBarrier

%% file: appendix_sections/add_results.tex
\subsection{Evaluation metrics}
\label{app:evaluation_metrics}

We evaluate long-horizon statistical fidelity using complementary metrics that emphasize different parts of the corrected distribution. Let \(p\) denote the ground truth density of a diagnostic observable and let \(q\) denote the corresponding density induced by a corrected trajectory. In the QG experiments, these densities are estimated from long test trajectories using the same binning and smoothing procedure for the ground truth, imperfect model baseline, STORN, and FiLM-STORN outputs.

The first metric is the Kullback--Leibler divergence
\begin{equation}
\label{eq:KL-divergence}
    D_{\mathrm{KL}}(p\|q)
    =
    \int p(x)\log\!\left(\frac{p(x)}{q(x)}\right)\,dx .
\end{equation}
KL divergence measures distributional mismatch, weighted by the reference or ground truth probability. Errors in the far tails may therefore contribute little to the total. Thus \(D_{\mathrm{KL}}\) is useful for assessing attractor-level consistency and the bulk of the invariant statistics, but it can be comparatively insensitive to very far-tail discrepancies when those regions have small probability mass.

To emphasize tail fidelity, we also report the \(L_1\) discrepancy between log densities,
\begin{equation}
\label{eq:logL1}
    L_{\log}(p,q)
    =
    \int_{\Omega_\rho}
    \left|
        \log(p(x)+\varepsilon)
        -
        \log(q(x)+\varepsilon)
    \right|\,dx .
\end{equation}
Here, \(\varepsilon>0\) is the numerical floor and \(\Omega_\rho\) is the region where the ground truth density exceeds the estimation cutoff. The logarithmic scale emphasizes relative density errors in low-probability regions. This makes \(L_{\log}\) a stricter diagnostic for rare-event emulation than ordinary density \(L_1\) error.

We further evaluate the spatial structure of extreme events through exceedance-area statistics. For a streamfunction field \(\psi(x,y,t)\) and amplitude threshold \(c\), define
\begin{equation}
\label{eq:exceedance_area}
    \frac{A_c(t)}{A}
    =
    \frac{1}{N_xN_y}
    \sum_{i=1}^{N_x}
    \sum_{j=1}^{N_y}
    \mathbf{1}\!\left\{
        |\psi(x_i,y_j,t)|>c
    \right\}.
\end{equation}

This observable measures the fraction of the spatial domain occupied by an extreme event at time \(t\). Unlike one-point marginal densities, \(A_c(t)/A\) is sensitive to the spatial extent, localization, and persistence of thresholded structures. It therefore tests whether the corrected model reproduces the geometry of rare excursions, not only their marginal amplitudes.

For each threshold \(c\) and layer \(\ell\), we compare the empirical distributions of \(A_{c,\ell}(t)/A\) using the one-dimensional Wasserstein-1 distance
\begin{equation}
\label{eq:aot_w1_error}
    \mathcal E_{\mathrm{AOT}}(c,\ell)
    =
    W_1\!\left(
    \mathcal L\{A^{\mathrm{ref}}_{c,\ell}(t)/A\}_{t=1}^{T_{\mathrm{test}}},
    \mathcal L\{A^{\mathrm{model}}_{c,\ell}(t)/A\}_{t=1}^{T_{\mathrm{test}}}
    \right).
\end{equation}
At high thresholds, the exceedance-area distribution often has substantial mass at zero. Histogram-based KL divergences or log-density errors can become unstable in this regime as bins may be empty. In contrast, \(W_1\) compares the empirical distributions directly and remains well-defined when one model produces too few or too many nonzero exceedance events.

We also report the event-frequency error
\begin{equation}
\label{eq:aot_frequency_error}
    \mathcal E_{\mathrm{freq}}(c,\ell)
    =
    \left|
    \mathbb P_{\mathrm{ref}}\!\left(A_{c,\ell}(t)>0\right)
    -
    \mathbb P_{\mathrm{model}}\!\left(A_{c,\ell}(t)>0\right)
    \right|.
\end{equation}
This metric isolates whether the model produces threshold exceedances at the correct rate, independently of their spatial size once they occur. The pair \((\mathcal E_{\mathrm{AOT}},\mathcal E_{\mathrm{freq}})\) separates two different rare-event errors. Namely, generating extremes too often or too rarely, and generating events with the wrong spatial extent.

We summarize the high-threshold behavior by averaging over $\mathcal C_{\mathrm{tail}}=\{1.0,1.5,1.75\}$, namely
\begin{equation}
\label{eq:aot_average_errors}
    \overline{\mathcal E}_{\mathrm{AOT}}(\ell)
    =
    \frac{1}{|\mathcal C_{\mathrm{tail}}|}
    \sum_{c\in\mathcal C_{\mathrm{tail}}}
    \mathcal E_{\mathrm{AOT}}(c,\ell),
    \qquad
    \overline{\mathcal E}_{\mathrm{freq}}(\ell)
    =
    \frac{1}{|\mathcal C_{\mathrm{tail}}|}
    \sum_{c\in\mathcal C_{\mathrm{tail}}}
    \mathcal E_{\mathrm{freq}}(c,\ell).
\end{equation}
The lower-layer statistic \(\psi_2\) is reported separately because the bottom topography acts directly in that layer. We also report the two-layer average to verify that the improvement is not an artifact of a single field.

\FloatBarrier

\subsection{Detailed results}\label{app:addresults:details}

This section provides the full set of QG diagnostics supporting the main empirical claims. Tables and figures here separate three distinct effects: global distributional fidelity (see e.g. Tables~\ref{tab:KL_abs} and~\ref{tab:logL1_abs}), tail accuracy, and spatial rare-event geometry. Importantly, a model can match the bulk PDF while missing the far tails, and it can match one-point tails while still producing thresholded events with the wrong spatial extent.
The main-text comparison uses the fixed configuration \(\sigma=0.01\), \(M=2\), and \(T_{\mathrm{train}}=50\), sweep-selected results are identified separately.

\begin{table}[H]
\centering
\caption{Global KL divergence, $D_{\mathrm{KL}}$ by training time trajectory length $T$, noise-level $\sigma$, and number of additional ensemble members $M$. Baseline is the unconditioned STORN model; the remaining columns are the FiLM-STORN including ensemble information. The reported values are averaged over five independent runs with negligible variation across runs.}
\small
\label{tab:KL_abs}
\setlength{\tabcolsep}{3pt}
\renewcommand{\arraystretch}{1.05}
\begin{tabular}{llcccccc}
\toprule
$T$ & baseline & $M$ & $\sigma=0.001$ & $\sigma=0.01$ & $\sigma=0.05$ & $\sigma=0.1$ & $\sigma=0.2$ \\
\midrule
\multirow{4}{*}{\(T=50\)} & \multirow{4}{*}{0.07079} & 1 & 0.036477 & 0.036453 & 0.057637 & 0.071921 & 0.78613 \\
 &  & 2 & 0.036464 & 0.036460 & 0.074342 & 0.063390 & 0.74884 \\
 &  & 4 & 0.036473 & 0.051943 & 0.087202 & 0.051904 & 0.65616 \\
 &  & 8 & 0.036442 & 0.053113 & 0.061234 & 0.071854 & 0.72615 \\
\midrule
\multirow{4}{*}{\(T=200\)} & \multirow{4}{*}{0.10011} & 1 & 0.032722 & 0.066066 & 0.087975 & 0.094977 & 0.083387 \\
 &  & 2 & 0.063504 & 0.045316 & 0.091538 & 0.15473 & 0.13318 \\
 &  & 4 & 0.052064 & 0.044252 & 0.11192 & 0.11831 & 0.091194 \\
 &  & 8 & 0.072961 & 0.059629 & 0.068314 & 0.11337 & 0.12425 \\
\midrule
\multirow{4}{*}{\(T=1000\)} & \multirow{4}{*}{0.0039339} & 1 & 0.015075 & 0.015913 & 0.0023566 & 0.0012746 & 0.023161 \\
 &  & 2 & 0.011283 & 0.010644 & 0.0023556 & 0.0025091 & 0.026599 \\
 &  & 4 & 0.016580 & 0.013081 & 0.0026269 & 0.0016688 & 0.026507 \\
 &  & 8 & 0.015483 & 0.014340 & 0.0013198 & 0.00092347 & 0.031131 \\
\bottomrule
\end{tabular}%

\end{table}

\begin{table}[t]
\centering
\caption{Global $L_{1_{\log}}$ by training time trajectory length $T$, noise-level $\sigma$, and number of additional ensemble members $M$. Baseline is the unconditioned STORN model, the remaining columns are the FiLM-STORN including ensemble information. Lower is better. The reported values are averaged over five independent runs with negligible variation across runs.}
\small
\label{tab:logL1_abs}
\setlength{\tabcolsep}{3pt}
\renewcommand{\arraystretch}{1.05}
\begin{tabular}{llcccccc}
\toprule
$T$ & baseline & $M$ & $\sigma=0.001$ & $\sigma=0.01$ & $\sigma=0.05$ & $\sigma=0.1$ & $\sigma=0.2$ \\
\midrule
\multirow{4}{*}{\(T=50\)} & \multirow{4}{*}{11.970} & 1 & 8.4110 & 8.4510 & 7.9457 & 9.5331 & 36.176 \\
 &  & 2 & 8.3122 & 8.2959 & 7.7896 & 8.8488 & 35.336 \\
 &  & 4 & 8.1277 & 12.310 & 6.6458 & 11.571 & 32.948 \\
 &  & 8 & 8.2253 & 15.228 & 8.4209 & 11.450 & 29.796 \\
\midrule
\multirow{4}{*}{\(T=200\)} & \multirow{4}{*}{13.828} & 1 & 7.5888 & 9.2690 & 11.450 & 13.215 & 12.623 \\
 &  & 2 & 9.4531 & 5.7103 & 14.582 & 18.820 & 12.733 \\
 &  & 4 & 11.598 & 12.026 & 13.172 & 18.443 & 12.773 \\
 &  & 8 & 10.515 & 9.7486 & 13.931 & 15.119 & 16.286 \\
\midrule
\multirow{4}{*}{\(T=1000\)} & \multirow{4}{*}{5.3454} & 1 & 9.3220 & 8.1347 & 6.1606 & 7.8986 & 9.9209 \\
 &  & 2 & 6.4275 & 6.7232 & 8.2007 & 5.5868 & 7.6797 \\
 &  & 4 & 10.557 & 7.1147 & 6.4902 & 4.8015 & 9.9834 \\
 &  & 8 & 7.0483 & 6.0650 & 5.2899 & 7.1667 & 8.0975 \\
\bottomrule
\end{tabular}%
\end{table}

\begin{table}[t]
\centering
\small
\caption{
Threshold-resolved exceedance-area diagnostics corresponding to Table~\ref{tab:aot_geometry_main}. 
For threshold \(c\), \(\mathcal E_{\mathrm{AOT}}(c)\) is the empirical Wasserstein-1 distance between the ground truth and model distributions of \(A_c(t)/A\), while \(\mathcal E_{\mathrm{freq}}(c)\) is the absolute error in event frequency \(\mathbb P(A_c(t)>0)\). Lower is better. FiLM-STORN denotes the ensemble-conditioned model trained with \(T_{\mathrm{train}}=50\), \(\sigma=0.01\), and \(M=2\). Bold entries mark the best value among the four reported models for each diagnostic and threshold.
}
\label{tab:aot_geometry_thresholds}
\setlength{\tabcolsep}{3.5pt}
\renewcommand{\arraystretch}{1.08}
\begin{tabular}{llcccccc}
\toprule
Diagnostic & Model
& \(\mathcal E_{\mathrm{AOT}}(1.0)\)
& \(\mathcal E_{\mathrm{AOT}}(1.5)\)
& \(\mathcal E_{\mathrm{AOT}}(1.75)\)
& \(\mathcal E_{\mathrm{freq}}(1.0)\)
& \(\mathcal E_{\mathrm{freq}}(1.5)\)
& \(\mathcal E_{\mathrm{freq}}(1.75)\) \\
\midrule
\multirow{4}{*}{\(\psi_1\)}
& CR
& 0.13839 & 0.01553 & 0.00389
& 0.70602 & 0.60802 & 0.34090 \\
& STORN \(T=50\)
& 0.06560 & 0.00861 & 0.00259
& 0.02953 & 0.15682 & \textbf{0.10090} \\
& STORN \(T=1000\)
& 0.04683 & 0.00766 & 0.00220
& 0.04715 & 0.20020 & 0.13932 \\
& FiLM-STORN \(T=50\)
& \textbf{0.04564} & \textbf{0.00399} & \textbf{0.00144}
& \textbf{0.00202} & \textbf{0.09209} & 0.13165 \\
\midrule
\multirow{4}{*}{\(\psi_2\)}
& CR
& 0.02876 & 0.00032 & \(2.67{\times}10^{-5}\)
& 0.72329 & 0.02798 & 0.00517 \\
& STORN \(T=50\)
& 0.01732 & 0.00031 & \(2.66{\times}10^{-5}\)
& 0.40766 & 0.02704 & 0.00511 \\
& STORN \(T=1000\)
& 0.01472 & 0.00028 & \(2.53{\times}10^{-5}\)
& 0.28833 & 0.02377 & 0.00481 \\
& FiLM-STORN \(T=50\)
& \textbf{0.01152} & \textbf{0.00020} & \textbf{\(1.98{\times}10^{-5}\)}
& \textbf{0.19028} & \textbf{0.01314} & \textbf{0.00278} \\
\midrule
\multirow{4}{*}{Avg. \(\psi_1,\psi_2\)}
& CR
& 0.08357 & 0.00792 & 0.00196
& 0.71466 & 0.31800 & 0.17303 \\
& STORN \(T=50\)
& 0.04146 & 0.00446 & 0.00131
& 0.21859 & 0.09193 & \textbf{0.05300} \\
& STORN \(T=1000\)
& 0.03078 & 0.00397 & 0.00111
& 0.16774 & 0.11199 & 0.07206 \\
& FiLM-STORN \(T=50\)
& \textbf{0.02858} & \textbf{0.00210} & \textbf{0.00073}
& \textbf{0.09615} & \textbf{0.05262} & {0.06722} \\
\bottomrule
\end{tabular}
\end{table}

The global density diagnostics show that ensemble conditioning is most useful in the data-limited regimes, but they do not by themselves establish recovery of spatial rare-event structure. We therefore report threshold-resolved exceedance-area diagnostics in Table~\ref{tab:aot_geometry_thresholds}. These metrics compare the empirical law of \(A_c(t)/A\), the fraction of the domain exceeding threshold \(c\), and are insensitive to time-phase mismatch after chaotic decorrelation.

\FloatBarrier
\begin{table}[t]
\centering
\caption{
Best FiLM-STORN configurations over the reported \((\sigma,M)\) sweep, evaluated using high-threshold averages over \(c\in\{1.0,1.5,1.75\}\). This table summarizes achievable performance under the sweep and is not used as the primary comparison, where \(\sigma=0.01\) and \(M=2\) are fixed for the \(T_{\mathrm{train}}=50\) FiLM-STORN model. $^*$ Best denotes the best configuration over the reported sweep, used here to summarize the achievable gain
}
\label{tab:aot_sweep_best}
\setlength{\tabcolsep}{5pt}
\renewcommand{\arraystretch}{1.12}
\begin{tabular}{ccccc}
\toprule
Diagnostic & \(T_{\mathrm{train}}\) & Metric & Best$^*$ \((\sigma,M)\) & Best$^*$ value \\
\midrule
\(\psi_2\) & 50
& \(\overline{\mathcal E}_{\mathrm{AOT}}\)
& \((0.1,8)\) & 0.00202 \\
\(\psi_2\) & 50
& \(\overline{\mathcal E}_{\mathrm{freq}}\)
& \((0.01,2)\) & 0.06874 \\
Avg. \(\psi_1,\psi_2\) & 50
& \(\overline{\mathcal E}_{\mathrm{AOT}}\)
& \((0.05,1)\) & 0.00862 \\
Avg. \(\psi_1,\psi_2\) & 50
& \(\overline{\mathcal E}_{\mathrm{freq}}\)
& \((0.1,4)\) & 0.06456 \\
\midrule
\(\psi_2\) & 200
& \(\overline{\mathcal E}_{\mathrm{AOT}}\)
& \((0.001,1)\) & 0.00344 \\
\(\psi_2\) & 200
& \(\overline{\mathcal E}_{\mathrm{freq}}\)
& \((0.001,1)\) & 0.07810 \\
Avg. \(\psi_1,\psi_2\) & 200
& \(\overline{\mathcal E}_{\mathrm{AOT}}\)
& \((0.001,1)\) & 0.00757 \\
Avg. \(\psi_1,\psi_2\) & 200
& \(\overline{\mathcal E}_{\mathrm{freq}}\)
& \((0.001,1)\) & 0.07409 \\
\midrule
\(\psi_2\) & 1000
& \(\overline{\mathcal E}_{\mathrm{AOT}}\)
& \((0.1,8)\) & 0.00171 \\
\(\psi_2\) & 1000
& \(\overline{\mathcal E}_{\mathrm{freq}}\)
& \((0.1,8)\) & 0.01940 \\
Avg. \(\psi_1,\psi_2\) & 1000
& \(\overline{\mathcal E}_{\mathrm{AOT}}\)
& \((0.1,8)\) & 0.00231 \\
Avg. \(\psi_1,\psi_2\) & 1000
& \(\overline{\mathcal E}_{\mathrm{freq}}\)
& \((0.05,8)\) & 0.01847 \\
\bottomrule
\end{tabular}
\end{table}

Table~\ref{tab:aot_sweep_best} reports the lowest error found for each diagnostic over the parameter sweep. These sweep-selected results are separate from the fixed main-text comparison. Its purpose is to show the achievable range and the existence of a broad useful regime. Very large perturbations, such as \(\sigma=0.2\), can degrade performance in the global metrics, consistent with the interpretation that the ensemble should probe local deformation geometry rather than act as a generic stochastic simulator.

\paragraph{Global and regional density visualizations.}
In addition to the scalar tables, we include log-PDF and regional PDF visualizations for the fixed short-data configuration used in the main comparison: \(T_{\mathrm{train}}=50\), \(\sigma=0.01\), and \(M=2\). The log-PDF plots expose tail discrepancies that are compressed on a linear density scale. The regional PDFs divide the spatial domain into subregions and test whether improvements are spatially localized or distributed across the flow. This is particularly important in the topographic QG setting, because the lower-layer topography induces spatially inhomogeneous extremes; a model may match the global PDF while still failing in the regions where rare events occur.

\begin{figure}[t]
\centering
\includegraphics[width=1.02\textwidth, trim=2.7cm 1.5cm 0cm 0cm, clip]{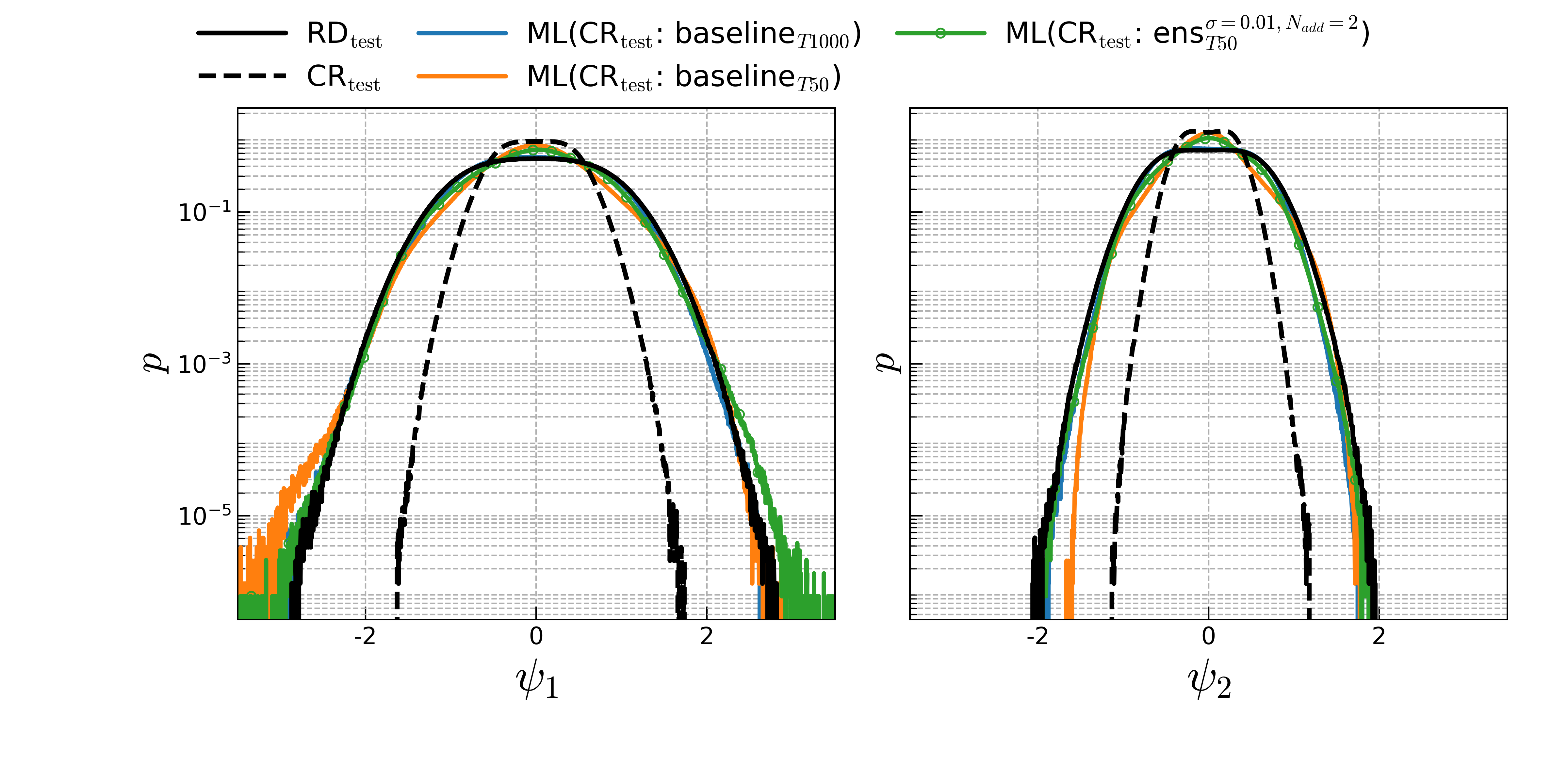}
\caption{
Global log-PDF comparison for the fixed short-data configuration \(T_{\mathrm{train}}=50\), \(\sigma=0.01\), and \(M=2\). The logarithmic scale highlights low-probability tails, where ordinary linear-scale PDFs can make substantial multiplicative errors appear visually small. Agreement in these plots supports tail-distribution fidelity, while the AOT diagnostics in Tables~\ref{tab:aot_geometry_thresholds}--\ref{tab:aot_sweep_best} test the corresponding spatial extent of thresholded events.
}
\label{fig:global_logpdf_fixed}
\end{figure}

\begin{figure}[t]
\centering
\includegraphics[width=1.0\textwidth,  trim=1.5cm 1.0cm 0cm 0cm, clip]{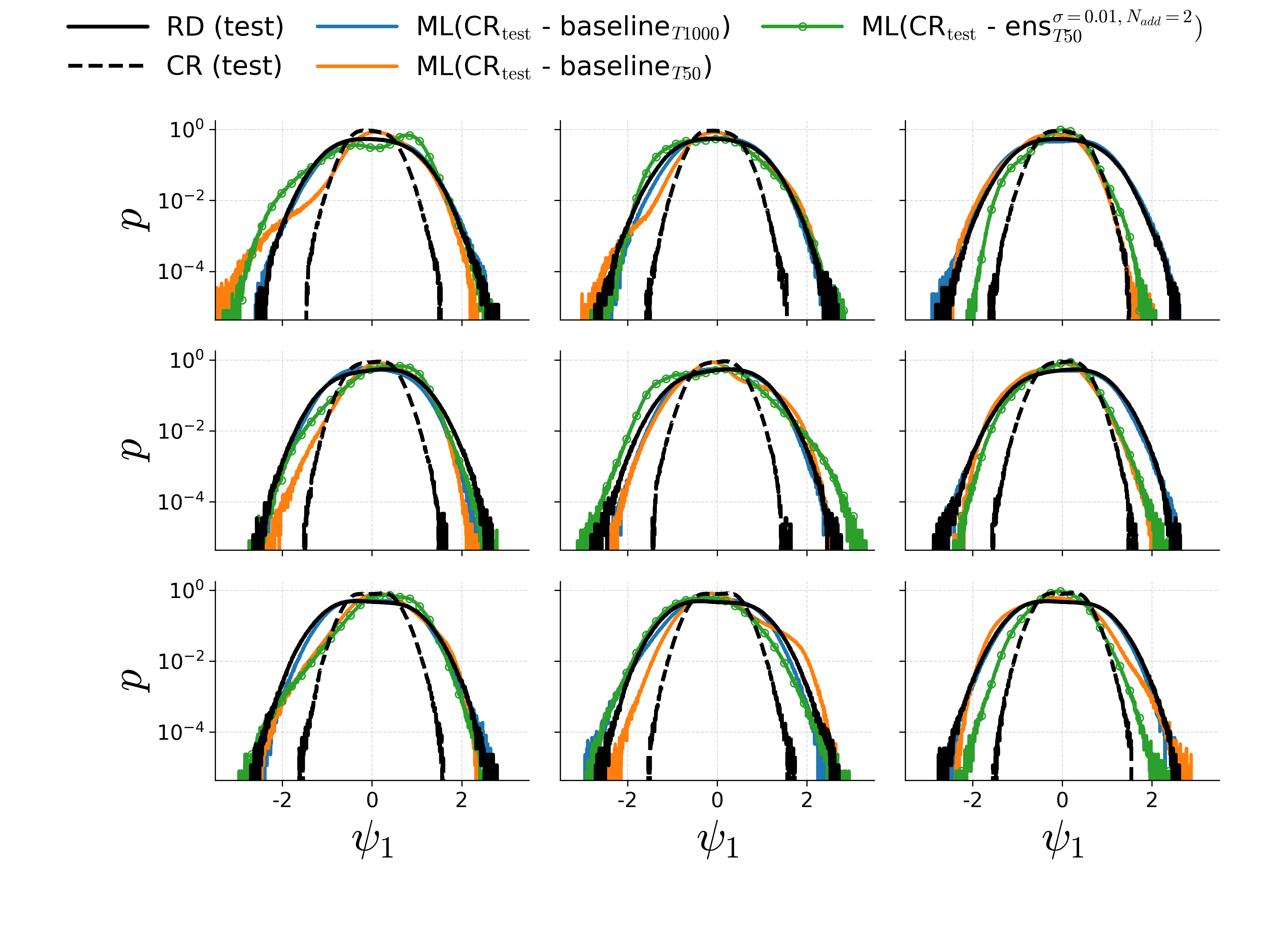}
\caption{
Regional log-PDF comparison for the upper streamfunction $\psi_1$ for the fixed short-data configuration \(T_{\mathrm{train}}=50\), \(\sigma=0.01\), and \(M=2\). The domain is partitioned into nine subregions, and each panel compares the projected ground truth, the uncorrected model trajectory, a STORN baseline trained on $T_{train}=1000$, a STORN baseline trained on $T_{train}=50$, and the FiLM-STORN correction trained on the same short-data configuration. These plots test whether the correction recovers spatially heterogeneous tail behavior rather than only matching a domain-averaged distribution.
}
\label{fig:regional_logpdf_fixed}
\end{figure}

\begin{figure}[t]
\centering
\includegraphics[width=1.0\textwidth,  trim=1.5cm 1.0cm 0cm 0cm, clip]{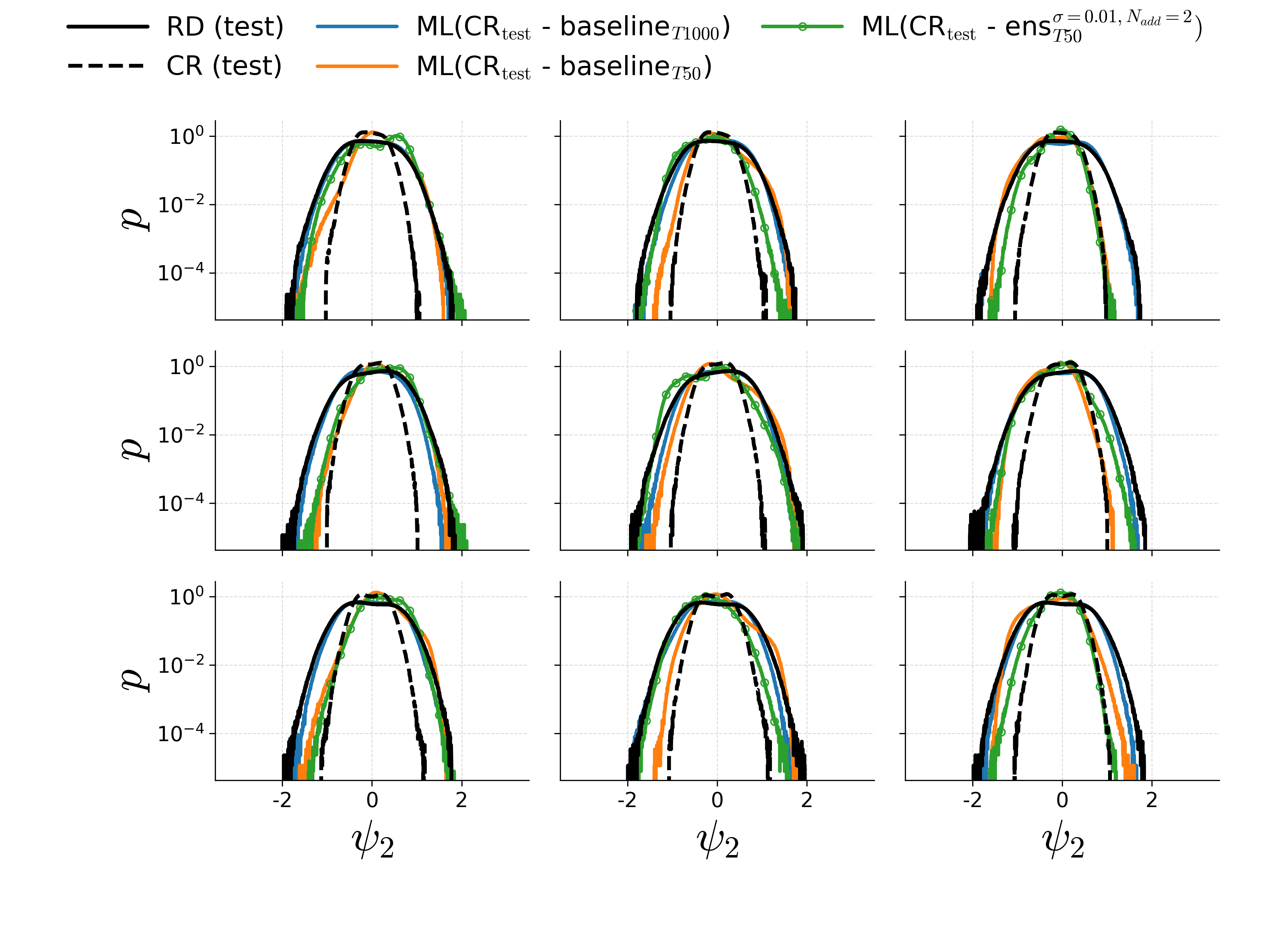}
\caption{
Regional log-PDF comparison for the lower streamfunction $\psi_2$ subject to topography for the fixed short-data configuration \(T_{\mathrm{train}}=50\), \(\sigma=0.01\), and \(M=2\). The domain is partitioned into nine subregions, and each panel compares the projected ground truth, the uncorrected and imperfect model trajectory, a STORN baseline trained on $T_{train}=1000$, a STORN baseline trained on $T_{train}=50$, and the FiLM-STORN correction trained on the same short-data configuration. These plots test whether the correction recovers spatially heterogeneous tail behavior rather than only matching a domain-averaged distribution.
}
\label{fig:regional_linearpdf_fixed}
\end{figure}

\FloatBarrier

%% file: checklist.tex
\newpage
\section*{NeurIPS Paper Checklist}

\begin{enumerate}

\item {\bf Claims}
    \item[] Question: Do the main claims made in the abstract and introduction accurately reflect the paper's contributions and scope?
    \item[] Answer: \answerYes{} 
    \item[] Justification: The introduction clearly lists the paper’s contribution.
    \item[] Guidelines:
    \begin{itemize}
        \item The answer \answerNA{} means that the abstract and introduction do not include the claims made in the paper.
        \item The abstract and/or introduction should clearly state the claims made, including the contributions made in the paper and important assumptions and limitations. A \answerNo{} or \answerNA{} answer to this question will not be perceived well by the reviewers. 
        \item The claims made should match theoretical and experimental results, and reflect how much the results can be expected to generalize to other settings. 
        \item It is fine to include aspirational goals as motivation as long as it is clear that these goals are not attained by the paper. 
    \end{itemize}

\item {\bf Limitations}
    \item[] Question: Does the paper discuss the limitations of the work performed by the authors?
    \item[] Answer: \answerYes{}{} 
    \item[] Justification: We discuss limitations in Section~\ref{sec:limitations}.
    \item[] Guidelines:
    \begin{itemize}
        \item The answer \answerNA{} means that the paper has no limitation while the answer \answerNo{} means that the paper has limitations, but those are not discussed in the paper. 
        \item The authors are encouraged to create a separate ``Limitations'' section in their paper.
        \item The paper should point out any strong assumptions and how robust the results are to violations of these assumptions (e.g., independence assumptions, noiseless settings, model well-specification, asymptotic approximations only holding locally). The authors should reflect on how these assumptions might be violated in practice and what the implications would be.
        \item The authors should reflect on the scope of the claims made, e.g., if the approach was only tested on a few datasets or with a few runs. In general, empirical results often depend on implicit assumptions, which should be articulated.
        \item The authors should reflect on the factors that influence the performance of the approach. For example, a facial recognition algorithm may perform poorly when image resolution is low or images are taken in low lighting. Or a speech-to-text system might not be used reliably to provide closed captions for online lectures because it fails to handle technical jargon.
        \item The authors should discuss the computational efficiency of the proposed algorithms and how they scale with dataset size.
        \item If applicable, the authors should discuss possible limitations of their approach to address problems of privacy and fairness.
        \item While the authors might fear that complete honesty about limitations might be used by reviewers as grounds for rejection, a worse outcome might be that reviewers discover limitations that aren't acknowledged in the paper. The authors should use their best judgment and recognize that individual actions in favor of transparency play an important role in developing norms that preserve the integrity of the community. Reviewers will be specifically instructed to not penalize honesty concerning limitations.
    \end{itemize}

\item {\bf Theory assumptions and proofs}
    \item[] Question: For each theoretical result, does the paper provide the full set of assumptions and a complete (and correct) proof?
    \item[] Answer: \answerYes{} 
    \item[] Justification: We proof all theoretical results in Appendix~\ref{app:proofs}.
    \item[] Guidelines:
    \begin{itemize}
        \item The answer \answerNA{} means that the paper does not include theoretical results. 
        \item All the theorems, formulas, and proofs in the paper should be numbered and cross-referenced.
        \item All assumptions should be clearly stated or referenced in the statement of any theorems.
        \item The proofs can either appear in the main paper or the supplemental material, but if they appear in the supplemental material, the authors are encouraged to provide a short proof sketch to provide intuition. 
        \item Inversely, any informal proof provided in the core of the paper should be complemented by formal proofs provided in appendix or supplemental material.
        \item Theorems and Lemmas that the proof relies upon should be properly referenced. 
    \end{itemize}

    \item {\bf Experimental result reproducibility}
    \item[] Question: Does the paper fully disclose all the information needed to reproduce the main experimental results of the paper to the extent that it affects the main claims and/or conclusions of the paper (regardless of whether the code and data are provided or not)?
    \item[] Answer: \answerYes{} 
    \item[] Justification: All the information used to run the experiments are clearly
stated in the main text, and further details are given in the corresponding appendices including evaluation metric and implementation details.
    \item[] Guidelines:
    \begin{itemize}
        \item The answer \answerNA{} means that the paper does not include experiments.
        \item If the paper includes experiments, a \answerNo{} answer to this question will not be perceived well by the reviewers: Making the paper reproducible is important, regardless of whether the code and data are provided or not.
        \item If the contribution is a dataset and\slash or model, the authors should describe the steps taken to make their results reproducible or verifiable. 
        \item Depending on the contribution, reproducibility can be accomplished in various ways. For example, if the contribution is a novel architecture, describing the architecture fully might suffice, or if the contribution is a specific model and empirical evaluation, it may be necessary to either make it possible for others to replicate the model with the same dataset, or provide access to the model. In general. releasing code and data is often one good way to accomplish this, but reproducibility can also be provided via detailed instructions for how to replicate the results, access to a hosted model (e.g., in the case of a large language model), releasing of a model checkpoint, or other means that are appropriate to the research performed.
        \item While NeurIPS does not require releasing code, the conference does require all submissions to provide some reasonable avenue for reproducibility, which may depend on the nature of the contribution. For example
        \begin{enumerate}
            \item If the contribution is primarily a new algorithm, the paper should make it clear how to reproduce that algorithm.
            \item If the contribution is primarily a new model architecture, the paper should describe the architecture clearly and fully.
            \item If the contribution is a new model (e.g., a large language model), then there should either be a way to access this model for reproducing the results or a way to reproduce the model (e.g., with an open-source dataset or instructions for how to construct the dataset).
            \item We recognize that reproducibility may be tricky in some cases, in which case authors are welcome to describe the particular way they provide for reproducibility. In the case of closed-source models, it may be that access to the model is limited in some way (e.g., to registered users), but it should be possible for other researchers to have some path to reproducing or verifying the results.
        \end{enumerate}
    \end{itemize}

\item {\bf Open access to data and code}
    \item[] Question: Does the paper provide open access to the data and code, with sufficient instructions to faithfully reproduce the main experimental results, as described in supplemental material?
    \item[] Answer: \answerYes{} 
    \item[] Justification:  We provide pseudo code for our algorithms. Codes to reproduce the results are open-source and available here:  \href{https://github.com/thisisisi/ExtremeEventSensor}{https://github.com/thisisisi/ExtremeEventSensor}.
    \item[] Guidelines:
    \begin{itemize}
        \item The answer \answerNA{} means that paper does not include experiments requiring code.
        \item Please see the NeurIPS code and data submission guidelines (\url{https://neurips.cc/public/guides/CodeSubmissionPolicy}) for more details.
        \item While we encourage the release of code and data, we understand that this might not be possible, so \answerNo{} is an acceptable answer. Papers cannot be rejected simply for not including code, unless this is central to the contribution (e.g., for a new open-source benchmark).
        \item The instructions should contain the exact command and environment needed to run to reproduce the results. See the NeurIPS code and data submission guidelines (\url{https://neurips.cc/public/guides/CodeSubmissionPolicy}) for more details.
        \item The authors should provide instructions on data access and preparation, including how to access the raw data, preprocessed data, intermediate data, and generated data, etc.
        \item The authors should provide scripts to reproduce all experimental results for the new proposed method and baselines. If only a subset of experiments are reproducible, they should state which ones are omitted from the script and why.
        \item At submission time, to preserve anonymity, the authors should release anonymized versions (if applicable).
        \item Providing as much information as possible in supplemental material (appended to the paper) is recommended, but including URLs to data and code is permitted.
    \end{itemize}

\item {\bf Experimental setting/details}
    \item[] Question: Does the paper specify all the training and test details (e.g., data splits, hyperparameters, how they were chosen, type of optimizer) necessary to understand the results?
    \item[] Answer: \answerYes{} 
    \item[] Justification: We provide a detailed description of the experimental methodology, reproducible code, additional results, proofs, and complete hyperparameter and simulation settings for the machine learning pipeline and trajectory generation. Supplementary details are provided in Appendix~\ref{app:lowdimsystem} for the low-dimensional example and in Appendices~\ref{app:QG} and~\ref{app:ML_QG} for the QG experiments.

    \item[] Guidelines:
    \begin{itemize}
        \item The answer \answerNA{} means that the paper does not include experiments.
        \item The experimental setting should be presented in the core of the paper to a level of detail that is necessary to appreciate the results and make sense of them.
        \item The full details can be provided either with the code, in appendix, or as supplemental material.
    \end{itemize}

\item {\bf Experiment statistical significance}
    \item[] Question: Does the paper report error bars suitably and correctly defined or other appropriate information about the statistical significance of the experiments?
    \item[] Answer: \answerYes{} 
    \item[] Justification: The two examples are evaluated using different task-specific metrics, whose interpretations are discussed in the corresponding sections ~\ref{app:lowdim_metrics} and \ref{app:evaluation_metrics}. All reported results are averaged over five independent runs, (see e.g. Tables \ref{tab:lowdim_z3_sweep_summary}, \ref{tab:lowdim_attention_ablation}, \ref{tab:KL_abs} and \ref{tab:logL1_abs}).
    \item[] Guidelines:
    \begin{itemize}
        \item The answer \answerNA{} means that the paper does not include experiments.
        \item The authors should answer \answerYes{} if the results are accompanied by error bars, confidence intervals, or statistical significance tests, at least for the experiments that support the main claims of the paper.
        \item The factors of variability that the error bars are capturing should be clearly stated (for example, train/test split, initialization, random drawing of some parameter, or overall run with given experimental conditions).
        \item The method for calculating the error bars should be explained (closed form formula, call to a library function, bootstrap, etc.)
        \item The assumptions made should be given (e.g., Normally distributed errors).
        \item It should be clear whether the error bar is the standard deviation or the standard error of the mean.
        \item It is OK to report 1-sigma error bars, but one should state it. The authors should preferably report a 2-sigma error bar than state that they have a 96\% CI, if the hypothesis of Normality of errors is not verified.
        \item For asymmetric distributions, the authors should be careful not to show in tables or figures symmetric error bars that would yield results that are out of range (e.g., negative error rates).
        \item If error bars are reported in tables or plots, the authors should explain in the text how they were calculated and reference the corresponding figures or tables in the text.
    \end{itemize}

\item {\bf Experiments compute resources}
    \item[] Question: For each experiment, does the paper provide sufficient information on the computer resources (type of compute workers, memory, time of execution) needed to reproduce the experiments?
    \item[] Answer: \answerYes{} 
    \item[] Justification: ompute resources are summarized for the low-dimensional experiment in Appendix \ref{app:lowdim:runtime}, and reported in more detail for the high-dimensional experiments in Appendix \ref{app:ml:impl}.
    \item[] Guidelines:
    \begin{itemize}
        \item The answer \answerNA{} means that the paper does not include experiments.
        \item The paper should indicate the type of compute workers CPU or GPU, internal cluster, or cloud provider, including relevant memory and storage.
        \item The paper should provide the amount of compute required for each of the individual experimental runs as well as estimate the total compute. 
        \item The paper should disclose whether the full research project required more compute than the experiments reported in the paper (e.g., preliminary or failed experiments that didn't make it into the paper). 
    \end{itemize}
    
\item {\bf Code of ethics}
    \item[] Question: Does the research conducted in the paper conform, in every respect, with the NeurIPS Code of Ethics \url{https://neurips.cc/public/EthicsGuidelines}?
    \item[] Answer: \answerYes{} 
    \item[] Justification: Our research is solely intended for scientific purposes, conforms in every respect with the NeurIPS Code of Ethics, and poses no potential ethical risks.
    \item[] Guidelines:
    \begin{itemize}
        \item The answer \answerNA{} means that the authors have not reviewed the NeurIPS Code of Ethics.
        \item If the authors answer \answerNo, they should explain the special circumstances that require a deviation from the Code of Ethics.
        \item The authors should make sure to preserve anonymity (e.g., if there is a special consideration due to laws or regulations in their jurisdiction).
    \end{itemize}

\item {\bf Broader impacts}
    \item[] Question: Does the paper discuss both potential positive societal impacts and negative societal impacts of the work performed?
    \item[] Answer: \answerNA{} 
    \item[] Justification: There is no societal impact of the work performed.
    \item[] Guidelines:
    \begin{itemize}
        \item The answer \answerNA{} means that there is no societal impact of the work performed.
        \item If the authors answer \answerNA{} or \answerNo, they should explain why their work has no societal impact or why the paper does not address societal impact.
        \item Examples of negative societal impacts include potential malicious or unintended uses (e.g., disinformation, generating fake profiles, surveillance), fairness considerations (e.g., deployment of technologies that could make decisions that unfairly impact specific groups), privacy considerations, and security considerations.
        \item The conference expects that many papers will be foundational research and not tied to particular applications, let alone deployments. However, if there is a direct path to any negative applications, the authors should point it out. For example, it is legitimate to point out that an improvement in the quality of generative models could be used to generate Deepfakes for disinformation. On the other hand, it is not needed to point out that a generic algorithm for optimizing neural networks could enable people to train models that generate Deepfakes faster.
        \item The authors should consider possible harms that could arise when the technology is being used as intended and functioning correctly, harms that could arise when the technology is being used as intended but gives incorrect results, and harms following from (intentional or unintentional) misuse of the technology.
        \item If there are negative societal impacts, the authors could also discuss possible mitigation strategies (e.g., gated release of models, providing defenses in addition to attacks, mechanisms for monitoring misuse, mechanisms to monitor how a system learns from feedback over time, improving the efficiency and accessibility of ML).
    \end{itemize}
    
\item {\bf Safeguards}
    \item[] Question: Does the paper describe safeguards that have been put in place for responsible release of data or models that have a high risk for misuse (e.g., pre-trained language models, image generators, or scraped datasets)?
    \item[] Answer: \answerNA{}{} 
    \item[] Justification: The paper does not pose such risks.
    \item[] Guidelines:
    \begin{itemize}
        \item The answer \answerNA{} means that the paper poses no such risks.
        \item Released models that have a high risk for misuse or dual-use should be released with necessary safeguards to allow for controlled use of the model, for example by requiring that users adhere to usage guidelines or restrictions to access the model or implementing safety filters. 
        \item Datasets that have been scraped from the Internet could pose safety risks. The authors should describe how they avoided releasing unsafe images.
        \item We recognize that providing effective safeguards is challenging, and many papers do not require this, but we encourage authors to take this into account and make a best faith effort.
    \end{itemize}

\item {\bf Licenses for existing assets}
    \item[] Question: Are the creators or original owners of assets (e.g., code, data, models), used in the paper, properly credited and are the license and terms of use explicitly mentioned and properly respected?
    \item[] Answer: \answerYes{}{} 
    \item[] Justification: We cite the original work and provide a link to the QG STORN baseline implementation in Appendix \ref{app:ml:model}.
    \item[] Guidelines:
    \begin{itemize}
        \item The answer \answerNA{} means that the paper does not use existing assets.
        \item The authors should cite the original paper that produced the code package or dataset.
        \item The authors should state which version of the asset is used and, if possible, include a URL.
        \item The name of the license (e.g., CC-BY 4.0) should be included for each asset.
        \item For scraped data from a particular source (e.g., website), the copyright and terms of service of that source should be provided.
        \item If assets are released, the license, copyright information, and terms of use in the package should be provided. For popular datasets, \url{paperswithcode.com/datasets} has curated licenses for some datasets. Their licensing guide can help determine the license of a dataset.
        \item For existing datasets that are re-packaged, both the original license and the license of the derived asset (if it has changed) should be provided.
        \item If this information is not available online, the authors are encouraged to reach out to the asset's creators.
    \end{itemize}

\item {\bf New assets}
    \item[] Question: Are new assets introduced in the paper well documented and is the documentation provided alongside the assets?
    \item[] Answer: \answerYes{} 
    \item[] Justification: We provide a readme file for how to run the code. 
    \item[] Guidelines:
    \begin{itemize}
        \item The answer \answerNA{} means that the paper does not release new assets.
        \item Researchers should communicate the details of the dataset\slash code\slash model as part of their submissions via structured templates. This includes details about training, license, limitations, etc. 
        \item The paper should discuss whether and how consent was obtained from people whose asset is used.
        \item At submission time, remember to anonymize your assets (if applicable). You can either create an anonymized URL or include an anonymized zip file.
    \end{itemize}

\item {\bf Crowdsourcing and research with human subjects}
    \item[] Question: For crowdsourcing experiments and research with human subjects, does the paper include the full text of instructions given to participants and screenshots, if applicable, as well as details about compensation (if any)? 
    \item[] Answer: \answerNA{} 
    \item[] Justification: {The paper does not involve crowdsourcing nor research with human subjects.}
    \item[] Guidelines:
    \begin{itemize}
        \item The answer \answerNA{} means that the paper does not involve crowdsourcing nor research with human subjects.
        \item Including this information in the supplemental material is fine, but if the main contribution of the paper involves human subjects, then as much detail as possible should be included in the main paper. 
        \item According to the NeurIPS Code of Ethics, workers involved in data collection, curation, or other labor should be paid at least the minimum wage in the country of the data collector. 
    \end{itemize}

\item {\bf Institutional review board (IRB) approvals or equivalent for research with human subjects}
    \item[] Question: Does the paper describe potential risks incurred by study participants, whether such risks were disclosed to the subjects, and whether Institutional Review Board (IRB) approvals (or an equivalent approval/review based on the requirements of your country or institution) were obtained?
    \item[] Answer: \answerNA{} 
    \item[] Justification: The paper does not involve crowdsourcing nor research with human subjects.
    \item[] Guidelines:
    \begin{itemize}
        \item The answer \answerNA{} means that the paper does not involve crowdsourcing nor research with human subjects.
        \item Depending on the country in which research is conducted, IRB approval (or equivalent) may be required for any human subjects research. If you obtained IRB approval, you should clearly state this in the paper. 
        \item We recognize that the procedures for this may vary significantly between institutions and locations, and we expect authors to adhere to the NeurIPS Code of Ethics and the guidelines for their institution. 
        \item For initial submissions, do not include any information that would break anonymity (if applicable), such as the institution conducting the review.
    \end{itemize}

\item {\bf Declaration of LLM usage}
    \item[] Question: Does the paper describe the usage of LLMs if it is an important, original, or non-standard component of the core methods in this research? Note that if the LLM is used only for writing, editing, or formatting purposes and does \emph{not} impact the core methodology, scientific rigor, or originality of the research, declaration is not required.
    \item[] Answer: \answerNA{} 
    \item[] Justification: The core method development in this research does not involve LLMs as any important, original, or non-standard components.
    \item[] Guidelines:
    \begin{itemize}
        \item The answer \answerNA{} means that the core method development in this research does not involve LLMs as any important, original, or non-standard components.
        \item Please refer to our LLM policy in the NeurIPS handbook for what should or should not be described.
    \end{itemize}

\end{enumerate}